\documentclass[journal]{IEEEtran}
\usepackage{blindtext}
\usepackage{graphicx}
\usepackage[utf8]{inputenc}
\usepackage{graphics} 
\usepackage{epsfig} 
\usepackage{mathptmx} 
\usepackage{amsmath} 
\usepackage{amssymb}  
\usepackage{cite}
\usepackage{multicol}
\usepackage{mathtools, cuted}
\usepackage[cmintegrals]{newtxmath}
\usepackage{epstopdf}
\usepackage{graphics} 
\usepackage{epsfig} 
\usepackage{subfigure}

\usepackage{titlesec}
\usepackage{epsfig} 
\usepackage{mathptmx} 
\usepackage{amsmath} 
\usepackage{blkarray, bigstrut}
\usepackage{amssymb}  
\usepackage{cite}
\usepackage{multicol}
\usepackage{mathtools, cuted}
\usepackage[cmintegrals]{newtxmath}
\usepackage{makeidx}         

\usepackage{algorithmic}

\DeclareMathOperator*{\argmax}{\arg\!\max}

\usepackage{multicol}        
\usepackage[bottom]{footmisc}
\usepackage{caption}

\usepackage{nomencl}
\usepackage[usenames, dvipsnames]{color}

 \usepackage{amsthm}
 \theoremstyle{definition}
\newtheorem{definition}{Definition}

 \theoremstyle{property}

\theoremstyle{theorem}
\newtheorem{theorem}{Theorem}

\theoremstyle{assumption}
\newtheorem{assumption}{Assumption}

\theoremstyle{lemma}

\theoremstyle{remark}

 \theoremstyle{proposition}
\newtheorem{proposition}{Proposition}

\usepackage{algorithmic}    

\usepackage[ruled,vlined]{algorithm2e}

\usepackage{array}
\usepackage{float}
\usepackage{graphicx}

\makeatletter
\newcommand*{\rom}[1]{\expandafter\@slowromancap\romannumeral #1@}
\makeatother

\ifCLASSINFOpdf
\else
\fi
\begin{document}
%
\title{Real-Time Deployment of a Large-Scale Multi-Quadcopter System (MQS)}
%
%
%

\author{Hossein Rastgoftar

\thanks{H. Rastgoftar is with the Department of Aerospace and Mechanical Engineering at the  University of Arizona, Tucson, AZ 85721, USA {Email: hrastgoftar@arizona.edu}}

}

%
%

\markboth{}%
{Shell \MakeLowercase{\textit{et al.}}: Bare Demo of IEEEtran.cls for IEEE Journals}
%



\maketitle

\begin{abstract}
This paper presents a continuum mechanics-based approach for real-time deployment (RTD) of a multi-quadcopter system between moving initial and final configurations arbitrarily distributed in a $3$-D motion space. The proposed RTD problem is decomposed into  spatial planning, temporal planning and acquisition sub-problems. For the spatial planning, the RTD desired coordination is defined by integrating (i) rigid-body rotation, (ii) one-dimensional homogeneous deformation, and (ii) one-dimensional heterogeneous coordination such that necessary conditions for inter-agent collision avoidance between every two quadcopter UAVs are satisfied.  By the RTD temporal planning, this paper suffices the inter-agent collision avoidance between every two individual quadcopters, and assures the boundedness of the rotor angular speeds for every individual quadcopter. For the RTD acquisition, each quadcopter modeled by a nonlinear dynamics applies a nonlinear control to stably and safely track the desired RTD trajectory such that the angular speeds of each quadcopter remain bounded and do not exceed a certain upper limit.
\end{abstract}


%
\IEEEpeerreviewmaketitle

\section{Introduction}
Over the past few decades, multi-agent coordination problems have been extensively studied and found numerous applications in surveillance \cite{thales2020single}, {\color{black}search and rescue} \cite{cooper2020optimal}, agricultural scouting \cite{skobelev2018designing}, structural health monitoring \cite{zhang2020participant}, and air traffic management \cite{idris2018air}.  Early work on multi-agent coordination commonly treats a group of vehicles (agents) as particles of a single rigid or deformable body acquiring the desired coordination in a centralized fashion  through leader-less \cite{ren2009distributed, li2020consensus, sun2019bipartite} or leader-follower communication-based \cite{notarstefano2011containment, li2021containment, wang2019necessary, xu2014necessary} approaches. More recently, researchers have studied real-time deployment (RTD) of multi-agent systems which is also called \textit{optimal mass transport} (OMT) in the literature. In the OMT problem,  agent coordination is governed by the continuity PDE and assigned by finding the optimal transformation between two arbitrary distributions with an equal mass \cite{villani2021topics, de2018optimal}. The existing OMT work assures convergence of agent deployment from an initial distribution to a target configuration. However, inter-agent collision avoidance may not be necessarily avoided when each individual agent represents an actual vehicle with finite size and nonlinear dynamics. This paper  develops a novel continuum-mechanics-based approach for collision-free real-time deployment of multi-vehicle system coordinating between two moving formations in a three-dimensional coordination space, where each vehicle represents a quadcopter modeled by a nonlinear dynamics.

\subsection{Related Work}
Early work on OMT was inspired by Schrodinger bridge problem \cite{di2020optimal, chen2016relation, haasler2021multimarginal} which was presented as transformation of the state density function from a reference configuration to a target configuration. Refs. \cite{chen2016relation, di2020optimal, haasler2021multimarginal} study the relation between the Schrodinger bridge problem and OMT problems. 
Mass transport  of linear systems from an initial configuration to an arbitrary target configuration is presented as an energy minimization optimization problem in Ref. \cite{chen2016optimal}. Furthermore, optimal transport of discrete-time  linear systems are studied in Refs. \cite{haasler2021multimarginal, hudoba2021discrete} where \cite{hudoba2021discrete} uses  linear quadratic Gaussian (LQG) regulation to formulate the OMT problem. This paper offers a continuum-mechanics-based solution to the OMT (RTD) problem which is inspired by the existing work on homogeneous transformation coordination of multi-agent systems presented in the author's previous work \cite{rastgoftar2021safe, rastgoftar2021scalable}. Because homogeneous transformation is an affine transformation, an $n$-D homogeneous transformation coordination can be defined as a decentralized leader-follower problem with $n+1$ that move independently and for and $n$-D simplex at any time $t$ and followers acquiring the desired coordination through local communication. 

\subsection{Contributions}
This paper applies the principles of kinematics of continuum mechanics to define RTD problem between arbitrary moving configurations  by combining (i) rigid-body rotation, (ii) $1$-D homogeneous transformation, and (iii) $2$-D heterogeneous coordination. This decomposition is advantageous since we can formally specify safety conditions, assure inter-agent collision avoidance, and impose the input constraints of individual  vehicles in a large-scale RTD problem. In this paper, we consider RTD of a multi-quadcopter system (MQS) and  define it  as spatial planning, temporal planning, and acquisition sub-problems. For the spatial planning, the RTD paths are determined between two moving configurations such that necessary conditions for inter-agent collision avoidance are provided. The RTD temporal planning determines the reference trajectories of individual quadcopters verifying all safety requirements. For the RTD acquisition, a low-level feedback linearization control is designed for each quadcopter such that the desired RTD trajectories are stably tracked and rotor angular speeds of all quadcopters remain bounded.

\subsection{Outline}
This paper is organized as follows: Preliminary notions are presented in Section \ref{Preliminaries} and followed by Problem Statement in Section \ref{Problem Statement}. The RTD is decomposed into planning and acquisition problems presented in Sections \ref{RTD Planning} and \ref{RTD Tracking}, respectively. The simulation results are presented in Section \ref{Simulation Results} and followed by Conclusion in Section \ref{Conclusion}.
\section{Preliminaries}\label{Preliminaries}
\subsection{Rigid-Body Rotation}


To realize position,  we define a global (an inertial) coordinate system, with base vectors $\hat{\mathbf{e}}_1$, $\hat{\mathbf{e}}_2$, and $\hat{\mathbf{e}}_3$, and  a local coordinate system with base vectors $\hat{\mathbf{c}}_1$, $\hat{\mathbf{c}}_2$, and $\hat{\mathbf{c}}_3$ that can rotate with respect to the inertial coordinate system. To characterize the rotation of the local coordinate system with respect to the global coordinate system, we first use the $3-2-1$ Euler angles standard to characterize a rigid-body rotation by 
\vspace{-.1cm}
\begin{equation}\label{Leuler}
    \resizebox{0.99\hsize}{!}{%
$
\mathbf{L}_{\mathrm{Euler}}\left(x_1,x_2,x_3\right)= \begin{bmatrix}
    C_{x_2} C_{x_3}&C_{x_2} S_{x_3} &-S_{x_2}\\
  S_{x_1}S_{x_2} C_{x_3}-C_{x_1}S_{x_3}&S_{x_1}S_{x_2} S_{x_3}+C_{x_1}C_{x_3}&S_{x_1}C_{x_2}\\
  C_{x_1}S_{x_2} C_{x_3}+S_{x_1}S_{x_3} &C_{x_1}S_{x_2} S_{x_3}-S_{x_1}C_{x_3}&C_{x_1}C_{x_2}
\end{bmatrix}
,
$
}
\end{equation}
where $x_1$, $x_2$, and $x_3$ are the first, second, and third Euler angles, respectively. Then, $\left(\hat{\mathbf{c}}_{1},\hat{\mathbf{c}}_{2},\hat{\mathbf{c}}_{3}\right)$ are related to $\left(\hat{\mathbf{e}}_{1},\hat{\mathbf{e}}_{2},\hat{\mathbf{e}}_{3}\right)$ by
\begin{equation}\label{rotationoflocal}
    \hat{\mathbf{c}}_{h}(\gamma(t),\mu(t))=\mathbf{R}_D(t)\hat{\mathbf{e}}_h,\qquad ~h=1,2,3,
\end{equation}
where rotation matrix $\mathbf{R}_D^T(t)=\mathbf{L}_{\mathrm{Euler}}\left(0,\gamma(t),\mu(t)\right)$ is given by
\begin{equation}\label{Rotation}
    \mathbf{R}_D(t)=
    \begin{bmatrix}
    \cos \gamma\cos \mu& \cos \gamma\sin \mu&-\sin\gamma\\
    -\sin\mu&\cos\mu&0\\
    \sin\gamma\cos\mu&\sin\gamma\sin\mu&\cos\gamma\\
    \end{bmatrix}
    .
\end{equation}
Note that the rotation matrix $\mathbf{R}_D(t)$ is defined based on the second and third Euler angles, where the first Euler angle is zero at any time $t$.

\subsection{Position Notation}
In this paper, we consider real-time deployment of an MQS consisting of $N$ quadcopters, where quadcopters' identification numbers are defined by $\mathcal{V}=\left\{1,\cdots,N\right\}$.
For every quadcopter $i\in \mathcal{V}$, we define the \textbf{global desired position} denoted by $\mathbf{p}_i=x_{i,g}\hat{\mathbf{e}}_1+y_{i,g}\hat{\mathbf{e}}_2+z_{i,g}\hat{\mathbf{e}}_3$ and the \textbf{local desired position} denoted by $\mathbf{a}_i=u_{i}\hat{\mathbf{c}}_1+v_{i}\hat{\mathbf{e}}_2+w_{i}\hat{\mathbf{e}}_3$, where $\mathbf{p}_i(t)$ and $\mathbf{a}_i(t)$ are related by
\begin{equation}
    \mathbf{p}_i(t)=\mathbf{d}(t)+\mathbf{a}_i(t),\qquad \forall i\in \mathcal{V},
\end{equation}
at time $t$, where $\mathbf{d}(t)=d_x(t)\hat{\mathbf{e}}_1+d_y(t)\hat{\mathbf{e}}_2+d_z(t)\hat{\mathbf{e}}_3$ is the rigid-body displacement vector assigning position of the origin of the local coordinate system. The \textbf{actual position} of quadcopter $i\in \mathcal{V}$ is denoted by $\mathbf{r}_i(t)$ and expressed with respect to the inertial coordinate system by $\mathbf{r}_i(t)=x_i(t)\hat{\mathbf{e}}_1+y_{i}(t)\hat{\mathbf{e}}_2+z_{i}(t)\hat{\mathbf{e}}_3$. This paper considers real-time deployment over the finite time interval $[t_s,t_f]$ where the \textbf{initial and final local positions} of every quadcopter $i\in \mathcal{V}$, denoted by $\mathbf{a}_{i,s}=\mathbf{a}_{i}(t_s)$ and $\mathbf{a}_{i,f}=\mathbf{a}_{i}(t_f)$, are known.

\begin{assumption}
We define unit vectors $\hat{\mathbf{e}}_1$, $\hat{\mathbf{e}}_2$, and $\hat{\mathbf{e}}_3$ as $\hat{\mathbf{e}}_1=\begin{bmatrix}1&0&0\end{bmatrix}^T$, $\hat{\mathbf{e}}_2=\begin{bmatrix}0&1&0\end{bmatrix}^T$, and $\hat{\mathbf{e}}_3=\begin{bmatrix}0&0&1\end{bmatrix}^T$. Therefore, rigid-body displacement, actual position, and global desired positions can be expressed in vector forms by $\mathbf{d}=\begin{bmatrix}d_x&d_y&d_z\end{bmatrix}^T$, $\mathbf{r}_i=\begin{bmatrix}x_{i}&y_{i}&z_{i}\end{bmatrix}^T$, and $\mathbf{p}_i=\begin{bmatrix}x_{i,g}&y_{i,g}&z_{i,g}\end{bmatrix}^T$, respectively. 
\end{assumption}

\section{Problem Statement}\label{Problem Statement}
We consider an  MQS consisting of $N$ quadcopters where dynamics of quadcopter $i\in \mathcal{V}$ is given by 
\begin{equation}\label{NonlinearDynamics}
    \begin{cases}
    \dot{\mathbf{x}}_i=\mathbf{f}\left(\mathbf{x}_i\right)+\mathbf{g}\left(\mathbf{x}_i\right)\mathbf{u}_i\\
    \mathbf{y}_i=\mathbf{C}\mathbf{x}_i
    \end{cases}
    ,\qquad \forall i\in \mathcal{V}.
\end{equation}
In \eqref{NonlinearDynamics}, $\mathbf{x}_i\in \mathbb{R}^{14}$ is the state vector; actual position $\mathbf{y}_i\in \mathbb{R}^4$ is the output vector; $\mathbf{u}_i\in \mathbb{R}^{4}$ is the input vector; $\mathbf{C}\in \mathbb{R}^{4\times 14}$ is constant; $\mathbf{f}:\mathbb{R}^{14}\rightarrow \mathbb{R}^{14} $ and $\mathbf{g}:\mathbb{R}^{14}\rightarrow \mathbb{R}^{14\times 4} $ are smooth functions;  $\mathbf{x}_i$, $\mathbf{u}_i$, $\mathbf{f}$, $\mathbf{g}$, and $\mathbf{C}$ are specified in Section \ref{RTD Tracking}. For every quadcopter $i\in \mathcal{V}$, we define the following properties and characteristics:
\begin{enumerate}
\item{\textbf{Agent size $\epsilon$:} Every quadcopter $i\in \mathcal{V}$ can be enclosed by a ball of radius $\epsilon$.}
\item{\textbf{Deviation upper bound $\delta$:} This paper assumes that each quadcopter can execute a proper trajectory tracking control such that the norm of tracking error is less than  $\delta $ for every quadcopter $i\in \mathcal{V}$ at any time $t$.}
\item{\textbf{Quadcopter rotor speed:} Angular speed of rotor $j\in \left\{1,2,3,4\right\}$ of quadcopter $i\in \mathcal{V}$ is denoted by $\varpi_{ij}$. The rotor angular speeds cannot exceed the upper bound $\varpi_{\mathrm{max}}$ for every quadcopter $i\in \mathcal{V}$.}
\item{\textbf{Admissible control set $\mathcal{U}$:} Control input $\mathbf{u}_i$, executed by quadcopter $i\in \mathcal{V}$, must belong to compact set $\mathcal{U}$, i.e. $u_i\in \mathcal{U},~\forall i\in \mathcal{V}$.}
\end{enumerate}

We assume that the  initial and final configurations of the quadcopter team  are arbitrarily distributed in the motion space and  defined by sets
\begin{subequations}
\begin{equation}
\label{IC}
    \Omega_s=\left\{\mathbf{a}_{i,s}=u_{i,s}\hat{\mathbf{c}}_{1,s}+v_{i,s}\hat{\mathbf{c}}_{2,s}+w_{i,s}\hat{\mathbf{c}}_{3,s},~\forall i\in \mathcal{V}\right\},
\end{equation}
\begin{equation}\label{FC}
     \Omega_f=\left\{\mathbf{a}_{i,f}=u_{i,f}\hat{\mathbf{c}}_{1,f}+v_{i,f}\hat{\mathbf{c}}_{2,f}+w_{i,f}\hat{\mathbf{c}}_{3,f},~\forall i\in \mathcal{V}\right\},
\end{equation}
\end{subequations}
where 
\begin{subequations}
\begin{equation}\label{initialc}
    \mathbf{a}_{i,s}=\mathbf{a}_i(t_s),\qquad \forall i\in \mathcal{V},
\end{equation}
\begin{equation}\label{finalc}
    \mathbf{a}_{i,f}=\mathbf{a}_i(t_f),\qquad \forall i\in \mathcal{V},
\end{equation}
\begin{equation}
    \hat{\mathbf{c}}_{h,s}=\hat{\mathbf{c}}_{h}\left(\gamma_s,\mu_s\right),\qquad h=1,2,3,
\end{equation}
\begin{equation}
    \hat{\mathbf{c}}_{h,f}=\hat{\mathbf{c}}_{h}\left(\gamma_f,\mu_f\right),\qquad h=1,2,3.
\end{equation}
\end{subequations}
Note that $\Omega_s$ and $\Omega_f$ are expressed with respect to the local coordinate system at times $t_s$ and $t_f$, respectively. However, $\gamma_s=\gamma\left(t_s\right)$, $\mu_s=\mu\left(t_s\right)$, $\gamma_f=\gamma\left(t_f\right)$, and $\mu_f=\mu\left(t_f\right)$ are assigned based on initial and target positions of the MQS, expressed with respect to the inertial coordinate system, by solving the following optimization problem:
\begin{subequations}
\begin{equation}\label{INITIAL}
    \left(\gamma_{s},\mu_s\right)=\argmax\limits_{\gamma,\mu\in \left[0,\pi\right)}\left(\min\limits_{i,h\in \mathcal{V},~i\neq h}\left|\left(\mathbf{p}_{i,s}-\mathbf{p}_{h,s}\right)\cdot \hat{\mathbf{c}}_1\left(\gamma,\mu\right)\right|\right),
\end{equation}
\begin{equation}\label{FINAL}
    \left(\gamma_{f},\mu_f\right)=\argmax\limits_{\gamma,\mu\in \left[0,\pi\right)}\left(\min\limits_{i,h\in \mathcal{V},~i\neq h}\left|\left(\mathbf{p}_{i,f}-\mathbf{p}_{h,f}\right)\cdot \hat{\mathbf{c}}_1\left(\gamma,\mu\right)\right|\right),
\end{equation}
\end{subequations}
where ``$\cdot$'' is the dot product symbol. Eqs. \eqref{INITIAL} and \eqref{FINAL} assign $\left(\gamma_{s},\mu_s\right)$ and $ \left(\gamma_{f},\mu_f\right)$ such that the minimum separation distance along unit vector $\hat{\mathbf{c}}_1$ is maximized.

Given above problem setting, the main objective of this paper is to \textbf{define} desired deployment trajectory $\mathbf{p}_i(t)$ and \textbf{choose} $\mathbf{u}_i(t)$, for every quadcopter $i\in \mathcal{V}$, such that initial condition \eqref{initialc}, final condition \eqref{finalc},  and the following safety conditions are all satisfied:
\begin{subequations}
\begin{equation}\label{quadcopterpropeller}
    \bigwedge_{i\in \mathcal{V}}\bigwedge_{j=1}^4\left(\left|\varpi_{ij}(t)\right|\leq \varpi_{\mathrm{max}}\right),\qquad \forall t\in \left[t_s,t_f\right],
\end{equation}
\begin{equation}\label{collisionavoidance}
        \bigwedge_{i=1}^{N-1}\bigwedge_{j=i+1}^{N}\left(\|\mathbf{r}_i(t)-\mathbf{r}_j(t)\|\geq 2\epsilon\right),\qquad \forall t\in \left[t_s,t_f\right],
    \end{equation}
    \begin{equation}\label{boundedness}
    \bigwedge_{i\in \mathcal{V}}\left(\|\mathbf{r}_{i}(t)-\mathbf{p}_i(t)\|\leq \delta\right),\qquad \forall t\in \left[t_s,t_f\right].
\end{equation}
\end{subequations}
Condition \eqref{quadcopterpropeller} assures that the angular speed of no rotor exceeds $\varpi_{\mathrm{max}}$. Eq. \eqref{collisionavoidance} specifies the inter-agent avoidance collision between every two quadcopters. Stable tracking condition is formally specified by Eq. \eqref{boundedness}. 

To assign $\mathbf{p}_i(t)$, for every quadcopter $i\in \mathcal{V}$, deployment of  the MQS--from arbitrary initial  $ {\Omega}_s$ to target configuration ${\Omega}_f$--is defined by integrating three collective motion modes: (i) rigid-body rotation, (ii) homogeneous ccordination, and (iii) heterogeneous coordination. Assuming  every quadcopter can satisfy safety condition \eqref{boundedness}, Section \ref{RTD Planning} provides guarantee conditions for inter-agent collision avoidance in a large-scale RTD.  The RTD planning is complemented with the RTD acquisition in Section \ref{RTD Tracking} where we apply a feedback linearization approach to design control $\mathbf{u}_i$, for every quadcopter $i\in \mathcal{V}$, such that: (i) quadcopter $i\in \mathcal{V}$ stably tracks the desired trajectory $\mathbf{p}_i(t)$ and safety conditions \eqref{quadcopterpropeller} and \eqref{boundedness} are both satisfied.

\section{RTD Planning}\label{RTD Planning}
Given the initial and final condition \eqref{initialc} and \eqref{finalc}, the desired position of quadcopter $i\in \mathcal{V}$, denoted by
\begin{equation}\label{RTDmain}
\begin{split}
    \mathbf{a}_i(t)=&u_i(t)\hat{\mathbf{c}}_1\left(\gamma(t),\mu(t)\right)+v_i(t)\hat{\mathbf{c}}_2\left(\gamma(t),\mu(t)\right)\\
    +&w_i(t)\hat{\mathbf{c}}_3\left(\gamma(t),\mu(t)\right),
\end{split}
\end{equation}
is planned under the assumption that the the desired formation of the MQS translates rigidly with constant velocities at the initial time $t_s$ and final time $t_f$.  This assumption can be satisfied, if:
\begin{subequations}
\begin{equation}\label{constantvelocity}
    \dot{\mathbf{a}}_i\left(t_s\right)= \dot{\mathbf{a}}_{i}\left(t_f\right)=0,\qquad \forall i\in \mathcal{V},
\end{equation}
\begin{equation}\label{constantacceleration}
    \ddot{\mathbf{a}}_i\left(t_s\right)=\ddot{\mathbf{a}}_{i}\left(t_f\right)=0,\qquad \forall i\in \mathcal{V}.
\end{equation}
\end{subequations}
Therefore, the global desired velocities of the quadcopters satisfy the following initial and final conditions:
\begin{subequations}
\begin{equation}\label{constantvelocity}
    \dot{\mathbf{p}}_i\left(t\right)= \dot{\mathbf{d}}\left(t_s\right)=\mathrm{constant},\qquad \forall i\in \mathcal{V},~t\leq t_s,
\end{equation}
\begin{equation}\label{constantacceleration}
    \dot{\mathbf{p}}_i\left(t_f\right)= \dot{\mathbf{d}}\left(t_f\right)=\mathrm{constant},\qquad \forall i\in \mathcal{V},~t\geq t_f.
\end{equation}
\end{subequations}


The RTD problem is spatially planned by combining three collective motion modes: (i) rigid-body rotation, (ii) homogeneous motion along $\hat{\mathbf{c}}_1$, and (iii) heterogeneous motion in $\hat{\mathbf{c}}_2-\hat{\mathbf{c}}_3$ plane. 

In Sections \ref{fold1}, \ref{fold2}, and \ref{fold3}, we use the quintic polynomial function
\vspace{-.3cm}
\begin{equation}
    \sigma(t,t_s,t_f)=15\left({t-t_s\over t_f-t_s}\right)^5-16\left({t-t_s\over t_f-t_s}\right)^4+10\left({t-t_s\over t_f-t_s}\right)^3
    , 
\end{equation}
for $t\in \left[t_s,t_f\right]$, to define the  collective motion modes in an RTD problem. Note that 
$\sigma(t_s,t_s,t_f)=0$, $\sigma(t_f,t_s,t_f)=1$, $\dot{\sigma}(t_s,t_s,t_f)=\dot{\sigma}(t_f,t_s,t_f)=0$, and $\ddot{\sigma}(t_s,t_s,t_f)=\ddot{\sigma}(t_f,t_s,t_f)=0$. \textit{Note that $\sigma(t,t_s,t_f)$ is strictly increasing with respect to $t$.}


\subsection{Rigid-Body Rotation}\label{fold1}
The orientation of the local coordinate system are assigned by using \eqref{rotationoflocal}, where
rotation matrix $\mathbf{R}_D(t)$, defined based on $\gamma(t)$ and $\mu(t)$ at any time $t\in \left[t_s,t_f\right]$, is given in \eqref{Rotation}. Therefore, rigid-body rotation of the MQS is specified by angles $\gamma(t)$ and $\mu(t)$ at any time $t\in \left[t_s,t_f\right]$. Given $\left(\gamma_s,\mu_s\right)$ and $\left(\gamma_f,\mu_f\right)$, we define
\begin{subequations}
\begin{equation}
   \gamma(t)=\gamma_{s}\left(1-\sigma(t,t_s,t_f)\right)+\gamma_{f}\sigma(t,t_s,t_f),
\end{equation}
\begin{equation}
   \mu(t)=\mu_{s}\left(1-\sigma(t,t_s,t_f)\right)+\mu_{f}\sigma(t,t_s,t_f),
\end{equation}
\end{subequations}
for $t\in \left[t_s,t_f\right]$.
\begin{figure*}
 \centering
 \subfigure[]{\includegraphics[width=0.32\linewidth]{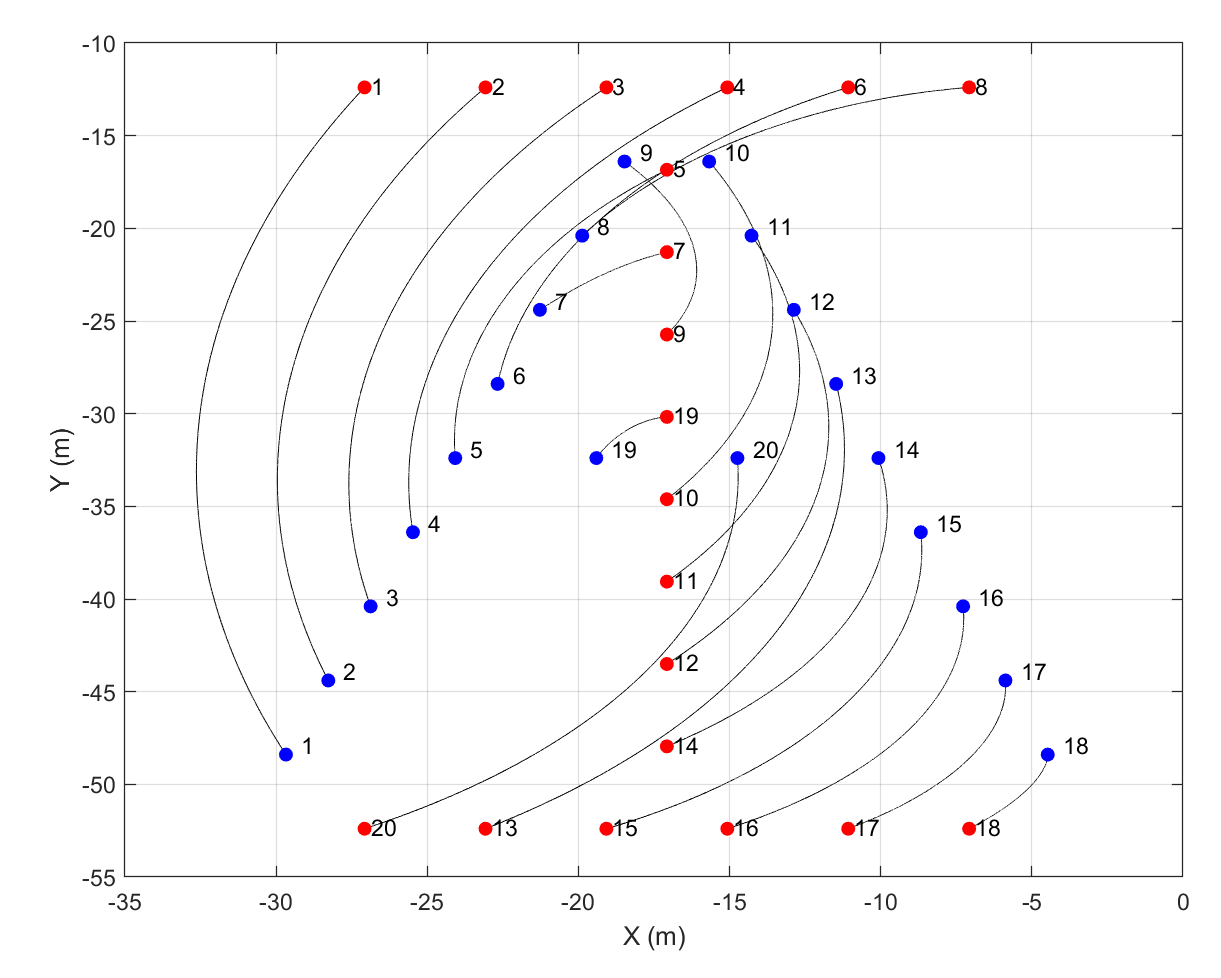}}
 \subfigure[]{\includegraphics[width=0.32\linewidth]{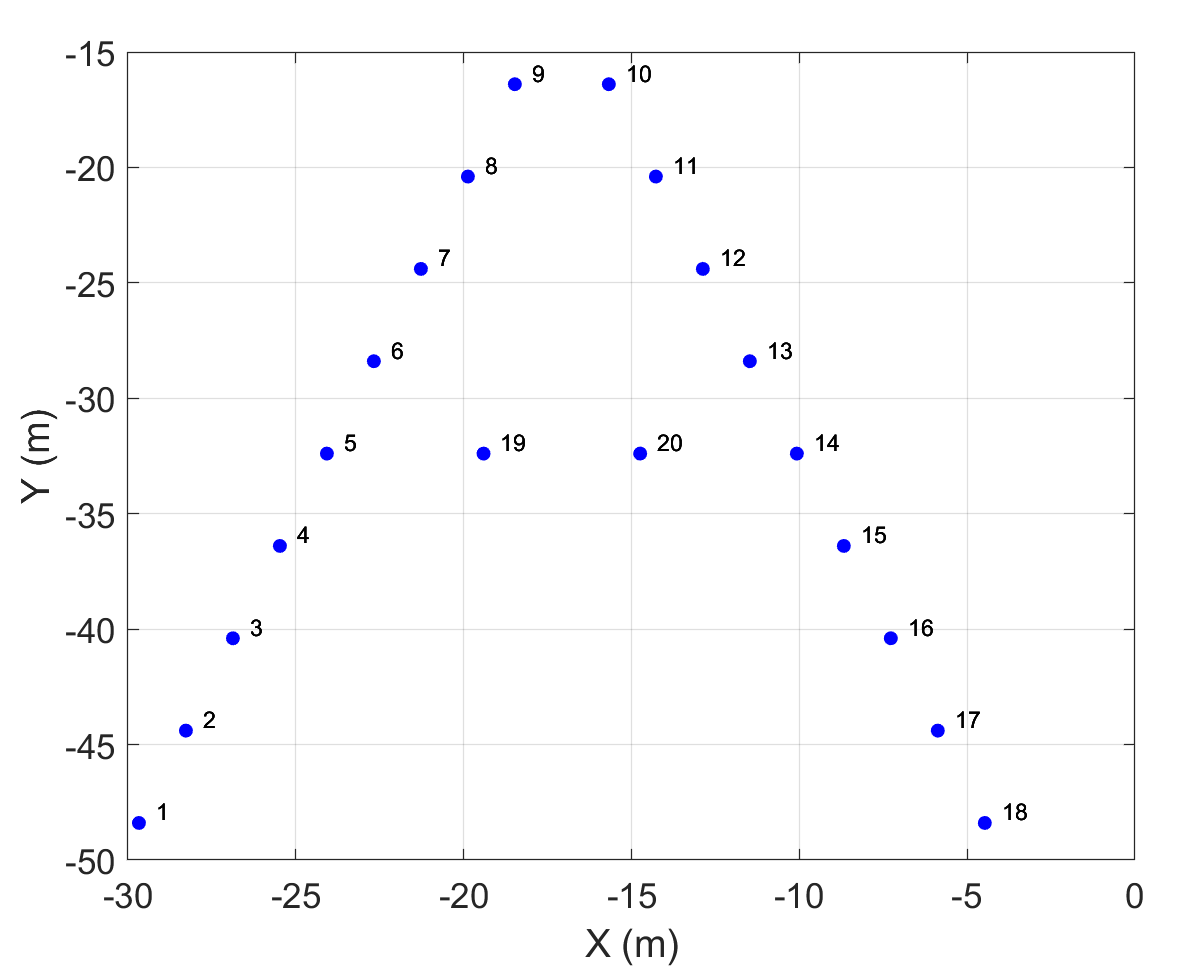}}
 \subfigure[]{\includegraphics[width=0.32\linewidth]{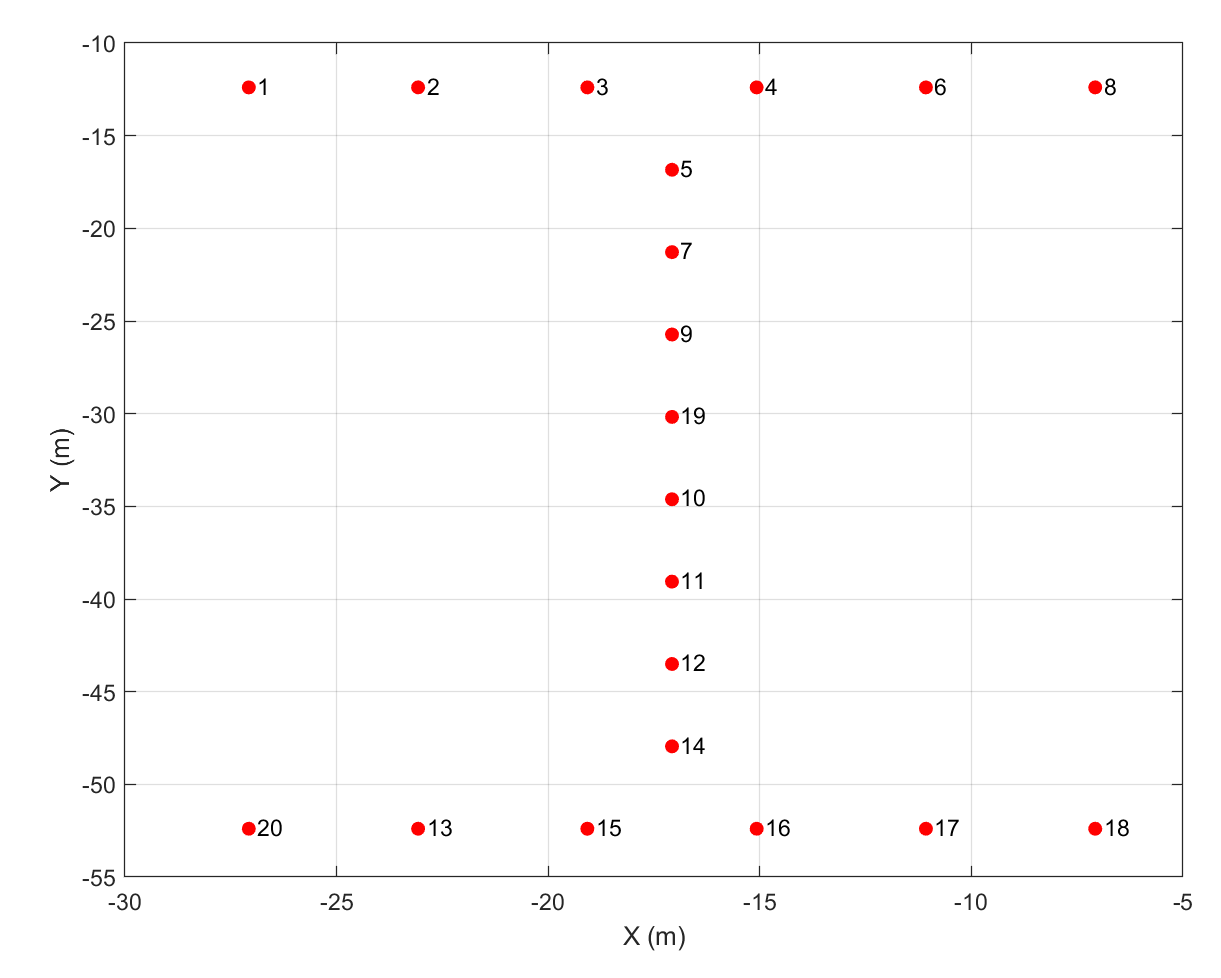}}
\vspace{-0.3cm}
     \caption{(a)  RTD paths from ``A" to ``I'' by 20 quadcopters distributed in the $x-y$ plane, i.e $\gamma(t)=0$ at any time $t\in \left[t_s,t_f\right]$. (b) Initial formation (Letter ``A''): $\mu_{s}=\mu\left(t_s\right)=172.8\deg$, $b_1=18$, $b_2=17$, $b_3=16$, $b_4=15$, $b_5=14$, $b_6=13$, $b_7=12$, $b_8=20$, $b_9=11$, $b_{10}=10$, $b_{11}=19$, $b_{12}= 9$, $ b_{13}=8$, $b_{14}=7$, $b_{15}= 6$, $b_{16}=5$, $b_{17}= 4$, $b_{18}= 3$, $b_{19}=2$, 
     $b_{20}=1$.  (c) Final formation (Letter ``I''): $\mu_{f}=\mu\left(t_f\right)=28.8\deg$.}
\label{DeploymentExample}
\end{figure*}

\subsection{Homogeneous Transformation Coordination along $\hat{\mathbf{c}}_1$}\label{fold2}
Given angles  $\left(\gamma_s,\mu_s\right)$ and  $\left(\gamma_f,\mu_f\right)$, assigned by solving \eqref{INITIAL} and \eqref{FINAL},  the initial and final configurations of the UAVs, denoted by $\Omega_s$ and $\Omega_f$,  are given by \eqref{IC} and \eqref{FC}, respectively. UAVs can be sorted based on their $u_{i,s}$ coordinates, along the unit vector $\hat{\mathbf{c}}_{1,s}$, and set $\mathcal{V}$ be expressed by
\begin{equation}
    \mathcal{V}=\left\{b_1,\cdots,b_{N}:u_{b_k,s}< u_{b_{k+1},s},~ k=1,\cdots,N-1 \right\}.
\end{equation}
where $b_k\in \mathcal{V}$ represents a quadcopter whose order number is $k$ in the initial formation ${\Omega}_s$. 
To assure inter-agent collision avoidance, we require that the order numbers of the quadcopters do not change when they are transforming from ${\Omega}_s$ to ${\Omega}_f$. Therefore, the RTD planning satisfies the following requirement:
\begin{equation}\label{mainhomogeneousss}
    \bigwedge_{k=1}^{N-1}\left(u_{b_k}(t)<u_{b_{k+1}}(t)\right),\qquad \forall t\in \left[t_s,t_f\right],~\mathcal{V}=\left\{b_1,\cdots,b_N\right\}.
\end{equation}
Per this requirement, the order numbers of the quadcopters in final configuration are the same as the order numbers of quadcopters in the initial configuration. Therefore, set $\mathcal{V}$ can be also defined as follows:
\begin{equation}
    \mathcal{V}=\left\{b_1,\cdots,b_{N}:u_{b_k,f}< u_{b_{k+1},f},~ k=1,\cdots,N-1 \right\},
\end{equation}
where $u_{b_k,f}=u_{b_k}(t_f)$.

\begin{definition}
Set $\mathcal{V}$ can be expressed as $\mathcal{V}=\mathcal{L}\bigcup\mathcal{F}$ where disjoint subsets \begin{subequations}
\begin{equation}
    \mathcal{L}=\left\{b_1,b_N\right\},
\end{equation}
\begin{equation}
    \mathcal{F}=\left\{b_2,\cdots,b_{N-1}\right\},
\end{equation}
\end{subequations}
define the \textit{leader quadcopters} and \textit{follower quadcopters}, respectively.
\end{definition}

For better clarification, consider the RTD example shown in Fig. \ref{DeploymentExample} that illustrates safe coordination of $20$ quadcopters from ``A" (initial formation) to ``I" (final formation) in the $x-y$ plane. Because the MQS is restricted to move in the $x-y$ plane, $\gamma(t)=0$ at any time $t\in  \left[t_s,t_f\right]$. Given initial formation of the MQS, we obtain $\mu_s=172.8\deg$, $\mathcal{L}=\left\{18,1\right\}$ ($b_1=18$, $b_{20}=1$) and $\mathcal{F}=\left\{1,\cdots,20\right\}\setminus \mathcal{L}$. Given the MQS final formation (``I"), $\mu_f=28.8\deg$ is obtained.
\begin{definition}
Given initial and final configurations of the MQS, we define initial reference weight $\beta_{i,s}$ and final reference weight $\beta_{i,f}$ for every quadcopter $i\in \mathcal{V}$  by
\begin{subequations}
\begin{equation}
    \beta_{i,s}={u_{b_{N},s}-u_{i,s}\over u_{b_{N},s}-u_{b_{1},s}}\in \left[0,1\right]
    ,
\end{equation}
\begin{equation}
    \beta_{i,f}={u_{b_{N},f}-u_{i,s}\over u_{b_{N},f}-u_{b_{1},f}}\in \left[0,1\right]
    .
\end{equation}
\end{subequations}
\end{definition}
Per definitions of set $\mathcal{F}$,  
$u_{b_1,s}< u_{i,s}<u_{b_N,s}$ and $u_{b_1,f}< u_{i,f}<u_{b_N,f}$. Therefore, $\beta_{i,s}>0$ and $\beta_{i,f}>0$ for every quadcopter $i\in \mathcal{F}$. Theorem \ref{thm1} provides guarantee conditions for satisfaction of  \eqref{mainhomogeneousss} through specifying initial and final MQS arrangements.
\begin{theorem}\label{thm1}
Define
\begin{subequations}
\begin{equation}\label{dmin}
    d_{\mathrm{min}}=\min\left\{u_{b_N,s}-u_{b_1,s},u_{b_N,f}-u_{b_1,s}\right\}
\end{equation}
\begin{equation}\label{betastar}
    \beta^*=\min\limits_{i,j\in \mathcal{V},~i\neq j}\min\left\{\left|\beta_{ij,s}\right|,\left|\beta_{ij,f}\right|\right\},
\end{equation}
\end{subequations}
where 
\begin{subequations}
\begin{equation}
    \beta_{ij,s}=\beta_{i,s}-\beta_{j,s},
\end{equation}
\begin{equation}
    \beta_{ij,f}=\beta_{i,f}-\beta_{j,f}.
\end{equation}
\end{subequations}
 Assume every quadcopter $i\in \mathcal{V}$ can execute a proper control input $\mathbf{u}_i$ such that safety condition \eqref{boundedness} is satisfied when every quadcopter is enclosed by a ball of radius $\epsilon$.  Then, inter-agent collision avoidance between every two quadcopters are avoided, if
 \begin{equation}\label{mainconditionthm1}
    d_{\mathrm{min}}\beta^*\geq 2\left(\delta+\epsilon\right),
\end{equation}
 and the $u_i$ component of  local desired position of every quadcopter $i\in \mathcal{V}$ is defined by
 \begin{equation}\label{ui}
     u_i(t)=
     \begin{cases}
     \left(1-\sigma\left(t,t_s,t_f\right)\right)u_{i,s}+\sigma\left(t,t_s,t_f\right)u_{i,f}&i\in \mathcal{L}\\
     \left(1-\beta_i\left(t\right)\right)u_{b_1}\left(t\right)+\beta_i\left(t\right)u_{b_N}\left(t\right)&i\in \mathcal{F}\\
     \end{cases}
 \end{equation}
at any time $t\in \left[t_s,t_f\right]$, where 
\begin{equation}\label{betai}
    \beta_i(t)=\left(1-\sigma\left(t,t_s,t_f\right)\right)\beta_{i,s}+\sigma\left(t,t_s,t_f\right)\beta_{i,f},\qquad \forall i\in \mathcal{V}.
\end{equation} 
\end{theorem}
\begin{proof}
Quadcopters $i$ and $j$ can both be enclosed by two balls with the same radius $\epsilon$ but different centers located at $\mathbf{r}_i(t)$ and $\mathbf{r}_j(t)$ ($\mathbf{r}_i(t)$ and $\mathbf{r}_j(t)$ are the actual position of quadcopters $i$ and $j$ at time $t$). If safety conditions \eqref{boundedness}, inter-agent collision avoidance can be assured by satisfying the following condition:
\begin{equation}\label{colavoid}
    \bigwedge_{i=1}^{N-1}\bigwedge_{j=i+1}^N\left(\left|u_i(t)-u_j(t)\right|\leq 2\left(\delta+\epsilon\right)\right),\qquad \forall t\in \left[t_s,t_f\right].
\end{equation}
When $u_i$ and $u_j$ coordinates of different quadcopters $i$ and $j$ are defined by \eqref{ui}, the following relation holds:
\begin{equation}\label{uiuj}
    u_i(t)-u_j(t)=\left(\beta_i(t)-\beta_j(t)\right)\left(u_{b_N}(t)-u_{b_1}(t)\right).
\end{equation}
Per Eq. \eqref{betai}, $\beta_i=\left(1-\sigma\right)\beta_{i,s}+\sigma\beta_{i,f}$ and $\beta_j=\left(1-\sigma\right)\beta_{j,s}+\sigma\beta_{j,f}$ can be substituted into Eq. \eqref{uiuj}; Eq. \eqref{betai} can be rewritten as follows:
\begin{equation}\label{uiujconverted}
\resizebox{0.99\hsize}{!}{%
$
    u_i(t)-u_j(t)=\left(\beta_{ij,s}+\sigma\left(t,t_s,t_f\right)\left(\beta_{ij,f}-\beta_{ij,s}\right)\right)\left(u_{b_N}(t)-u_{b_1}(t)\right).
    $
    }
\end{equation}
Because $\sigma\left(t,t_s,t_f\right)$ is strictly increasing over $\left[t_s,t_f\right]$, the right-hand side of Eq. \eqref{uiujconverted} reaches its minimum value, over $\left[t_s,t_f\right]$, either $t=t_s$, when $\sigma=0$, or $t=t_f$, when $\sigma=1$:
\[
\resizebox{0.99\hsize}{!}{%
$
\min\limits_{t\in \left[t_s,t_f\right]}\left|\beta_{ij,s}+\sigma\left(t,t_s,t_f\right)\left(\beta_{ij,f}-\beta_{ij,s}\right)\right|=\min\left\{\left|\beta_{ij,s}\right|,\left|\beta_{ij,f}\right|\right\}=\beta^*,
$
}
\]
\[
\resizebox{0.99\hsize}{!}{%
$
\min\limits_{t\in \left[t_s,t_f\right]}\left(u_{b_N}(t)-u_{b_1}(t)\right)=\min\left\{u_{b_N,f}-u_{b_1,f},u_{b_N,s}-u_{b_1,s}\right\}=d_{\mathrm{min}}.
$
}
\]
This implies that
\[
\min\limits_{t\in \left[t_s,t_f\right]}\left(u_i(t)-u_j(t)\right)\geq d_{\mathrm{\min}}\beta^*,\qquad i\neq j,~i,j\in \mathcal{V}.
\]
Therefore, inter-agent collision avoidance \eqref{colavoid} is satisfied, if condition \eqref{mainconditionthm1} holds.

\end{proof}

\subsection{Heterogeneous Transformation Coordination in the $\hat{\mathbf{c}}_2-\hat{\mathbf{c}}_3$}\label{fold3}
Evolution of the UAVs in the plane made by $\hat{\mathbf{c}}_2$ and $\hat{\mathbf{c}}_2$ are defined  by 
\begin{equation}
\resizebox{0.99\hsize}{!}{%
$
   \begin{bmatrix}
   v_i(t)\\
   w_i(t)\\
   \end{bmatrix}
   = \left(1-\sigma(t,t_s,t_f)\right)\begin{bmatrix}
   v_{i,s}\\
   w_{i,s}\\
   \end{bmatrix}
  +\sigma(t,t_s,t_f)\begin{bmatrix}
   v_{i,f}\\
   w_{i,f}\\
   \end{bmatrix}
   ,\qquad \forall t\in \left[t_s,t_f\right].
   $
   }
\end{equation}


\section{RTD Acquisition}\label{RTD Tracking}
We first present the quadcopter dynamics in Section \eqref{Quadcopter Dynamics}. Then, we design a feedback linearization control in Section \ref{Quadcopter Trajectory Control} so that every quadcopter $i$ can stably  track the desired RTD trajectory $\mathbf{p}_i(t)$ and  safety conditions  \eqref{quadcopterpropeller}-\eqref{boundedness} are all satisfied.

\subsection{Quadcopter Dynamics}
\label{Quadcopter Dynamics}
This paper models quadcopter $i\in \mathcal{V}$ by dynamics \eqref{NonlinearDynamics} with the state vector $\mathbf{x}_i$ and input vector $\mathbf{u}_i$, and smooth functions $\mathbf{f}$ and $\mathbf{g}$ defined as follows:
\begin{subequations}
\label{xufg}
\begin{equation}
\resizebox{0.99\hsize}{!}{%
$
    \mathbf{x}_i=\begin{bmatrix}
x_i&y_i&z_i&\dot{x}_i&\dot{y}_i&\dot{z}_i&\phi_i&\theta_i&\psi_i&\dot{\phi}_i&\dot{\theta}_i&\dot{\psi}_i&p_i&\dot{p}_i
\end{bmatrix}
^T,
$
}
\end{equation}
\begin{equation}
    \mathbf{u}_i=\begin{bmatrix}
u_{p,i}&u_{\phi,i}&u_{\theta,i}&u_{\psi,i}
\end{bmatrix}
^T,
\end{equation}
\begin{equation}
    \mathbf{f}\left(\mathbf{x}_i\right)=\begin{bmatrix}
    \dot{\mathbf{r}}_i^T&\left({p_i\over m_i}\hat{\mathbf{k}}_{b,i}-g\hat{\mathbf{e}}_3\right)^T&\dot{\phi}_i&\dot{\theta}_i&\dot{\psi}_i&\mathbf{0}_{1\times 3}&\dot{p}_i&0
    \end{bmatrix}
    ^T
,
\end{equation}
\begin{equation}
   {\mathbf{g}}\left({\mathbf{x}}_i\right)=
\begin{bmatrix}
\mathbf{0}_{9\times 1}&\mathbf{0}_{9\times 3}\\
\mathbf{0}_{3\times 1}&\mathbf{I}_3\\
0&\mathbf{0}_{1\times 3}\\
1&\mathbf{0}_{1\times 3}\\
\end{bmatrix}
,
\end{equation}
\begin{equation}
  \mathbf{C}=\begin{bmatrix}
  \mathbf{I}_3&\mathbf{0}_{3\times 5}&\mathbf{0}_{3\times 1}&\mathbf{0}_{3\times 5}\\
  \mathbf{0}_{1\times 3}&\mathbf{0}_{1\times 5}&1&\mathbf{0}_{1\times 5}
  \end{bmatrix}
 ,
\end{equation}
\end{subequations}
where  $\mathbf{r}_i=\begin{bmatrix}
x_i&y_i&z_i
\end{bmatrix}^T$ is the actual position of quadcopter $i\in \mathcal{V}$; $\phi_i$, $\theta_i$, and $\psi_i$ are the roll, pitch, and yaw angles of quadcopter $i\in \mathcal{V}$; $p_i$ is the magnitude of the thrust force of quadcopter $i\in \mathcal{V}$; $m_i$ is the mass of quadcopter $i\in \mathcal{V}$, and $g=9.81m/s^2$ is the gravity acceleration. Also, unit $\hat{\mathbf{k}}_{b,i}$ is the unit vector assigning  the direction of the thrust force of quadcopter  $i\in \mathcal{V}$. 
\begin{figure}[ht]
\centering
\includegraphics[width=3.3   in]{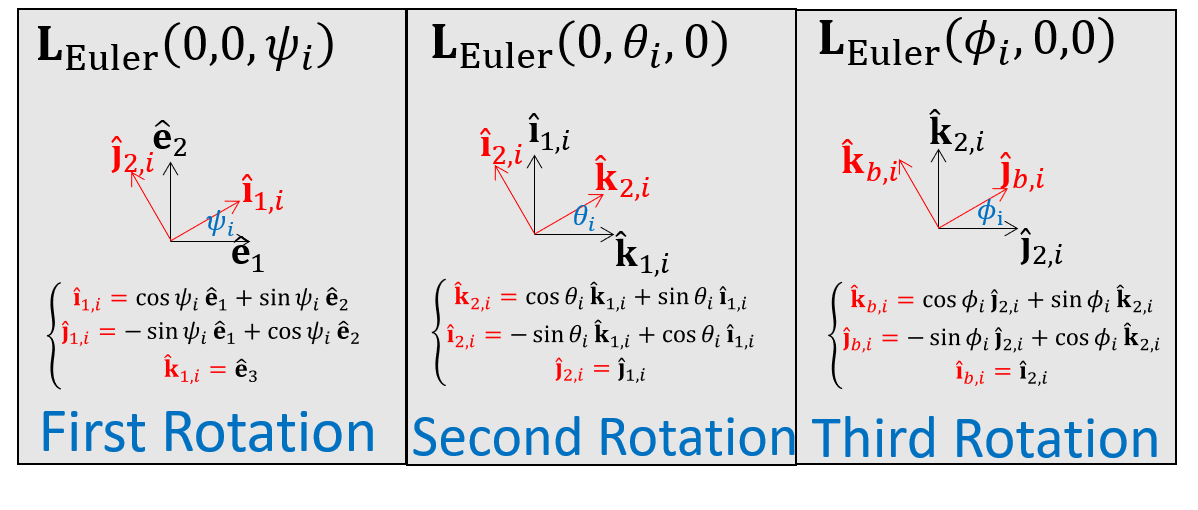}
\caption{Determination of rotation of quadcopter $i\in \mathcal{V}$ using the $3-2-1$ standard with roll angle $\phi_i$, pitch angle $\theta_i$, and yaw angle $\psi_i$.
}
\label{EulerAngleRotations}
\end{figure}
\subsubsection{Quadcopters' Angular Velocities and Accelerations}
We use 3-2-1 standard to determine orientation of quadcopter $i\in \mathcal{V}$ at time $t$ with the three rotations shown in Fig. \ref{EulerAngleRotations}. Given roll angle $\phi_i(t)$, pitch angle $\theta_i(t)$, and yaw angle $\psi_i(t)$ and the base vectors of the inertial coordinate system ($\hat{\mathbf{e}}_1$, $\hat{\mathbf{e}}_2$, and $\hat{\mathbf{e}}_3$), we obtain $\left(\hat{\mathbf{i}}_{1,i},\hat{\mathbf{j}}_{1,i},\hat{\mathbf{k}}_{1,i}\right)$, $\left(\hat{\mathbf{i}}_{2,i},\hat{\mathbf{j}}_{2,i},\hat{\mathbf{k}}_{2,i}\right)$, and $\left(\hat{\mathbf{i}}_{b,i},\hat{\mathbf{j}}_{b,i},\hat{\mathbf{k}}_{b,i}\right)$ as follows:
\begin{subequations}
\begin{equation}
        \hat{\mathbf{i}}_{1,i}
    =\mathbf{L}_{\mathrm{Euler}}^T\left(0,0,\psi_i\right)\hat{\mathbf{e}}_1=\begin{bmatrix}
    C_{\Psi_i}&S_{\psi_i}&0
    \end{bmatrix}^T,
\end{equation}
\begin{equation}
        \hat{\mathbf{j}}_{1,i}
    =\mathbf{L}_{\mathrm{Euler}}^T\left(0,0,\psi_i\right)\hat{\mathbf{e}}_2=\begin{bmatrix}
    -S_{\Psi_i}&C_{\psi_i}&0
    \end{bmatrix}^T,
\end{equation}
\begin{equation}
        \hat{\mathbf{k}}_{1,i}
    =\mathbf{L}_{\mathrm{Euler}}^T\left(0,0,\psi_i\right)\hat{\mathbf{e}}_3=\begin{bmatrix}
   0&0&1
    \end{bmatrix}^T,
\end{equation}
\end{subequations}
\begin{subequations}
\begin{equation}
\begin{split}
    \hat{\mathbf{i}}_{2,i}&
    =\mathbf{L}_{\mathrm{Euler}}^T\left(0,\theta_i,\psi_i\right)\hat{\mathbf{e}}_1=\begin{bmatrix}
    C_{\theta_i}C_{\psi_i}&C_{\theta_i}S_{\psi_i}&- S_{\theta_i}
    \end{bmatrix}
    ^T,
\end{split}
\end{equation}
\begin{equation}
\begin{split}
    \hat{\mathbf{j}}_{2,i}&
    =\mathbf{L}_{\mathrm{Euler}}^T\left(0,\theta_i,\psi_i\right)\hat{\mathbf{e}}_2=\begin{bmatrix}
    -S_{\psi_i}&C_{\psi_i}&0
    \end{bmatrix}
    ^T,
\end{split}
\end{equation}
\begin{equation}
       \begin{split}
    \hat{\mathbf{k}}_{2,i}&
    =\mathbf{L}_{\mathrm{Euler}}^T\left(0,\theta_i,\psi_i\right)\hat{\mathbf{e}}_3=\begin{bmatrix}
    S_{\theta_i}C_{\psi_i}&S_{\theta_i}S_{\psi_i}&C_{\theta_i}
    \end{bmatrix}
    ^T,
\end{split}
\end{equation}
\end{subequations}
\begin{subequations}
\begin{equation}
        \hat{\mathbf{i}}_{b,i}=\mathbf{L}_{\mathrm{Euler}}^T\left(\phi_i,\theta_i,\psi_i\right)\hat{\mathbf{e}}_1=\begin{bmatrix}
    C_{\theta_i}C_{\psi_i}&C_{\theta_i}S_{\psi_i}&- S_{\theta_i}
    \end{bmatrix}
    ^T,
\end{equation}
\begin{equation}
\begin{split}
            \hat{\mathbf{j}}_{b,i}
    =&\mathbf{L}_{\mathrm{Euler}}^T\left(\phi_i,\theta_i,\psi_i\right)\hat{\mathbf{e}}_2\\
    =&\begin{bmatrix}
    C_{\psi_i} S_{\phi_i} S_{\theta_i} - C_{\phi_i}S_{\psi_i}&C_{\phi_i}C_{\psi_i} +S_{\phi_i}S_{\psi_i}S_{\theta_i}& C_{\theta_i}S_{\phi_i}
    \end{bmatrix}^T,
\end{split}
\end{equation}
\begin{equation}
  \begin{split}
            \hat{\mathbf{k}}_{b,i}
    =&\mathbf{L}_{\mathrm{Euler}}^T\left(\phi_i,\theta_i,\psi_i\right)\hat{\mathbf{e}}_3\\
    =&\begin{bmatrix}
    S_{\phi_i}S_{\psi_i} +  C_{\phi_i}C_{\psi_i}S_{\theta_i}&C_{\phi_i}S_{\psi_i}S_{\theta_i}- S_{\phi_i}C_{\psi_i}&C_{\phi_i}C_{\theta_i}
    \end{bmatrix}^T.
\end{split}
\end{equation}
\end{subequations}
The angular velocity of quadcopter $i\in \mathcal{V}$ is then given by
\begin{equation}
\label{AngularVelocity}
    {\color{black}\bf{\omega}}_i=\dot{\psi}_i\hat{\mathbf{k}}_{1,i}+\dot{\theta}_i\hat{\mathbf{j}}_{2,i}+\dot{\phi}_i\hat{\mathbf{i}}_{b,i}.
\end{equation}
Substituting $\hat{\mathbf{k}}_{1,i}$, $\hat{\mathbf{j}}_{2,i}$, and $\hat{\mathbf{i}}_{b,i}$ into Eq. \eqref{AngularVelocity}, ${\color{black}{\omega}}_i=\left[\omega_{x,i}~\omega_{y,i}~\omega_{z,i}\right]^T$ is related by $\dot{\phi}_i$, $\dot{\theta}_i$, and $\dot{\psi}_i$ by
\begin{equation}
    \begin{bmatrix}
    \omega_{x,i}&
    \omega_{y,i}&
    \omega_{z,i}
    \end{bmatrix}^T
    =\mathbf{\Gamma}\left(\phi_i,\theta_i,\psi_i\right)
    \begin{bmatrix}
        \dot{\phi}_{i}&
        \dot{\theta}_{i}&
        \dot{\psi}_{i}
    \end{bmatrix}
    ^T
    ,
\end{equation}
where
\begin{equation}
\label{Eq55}
\mathbf{\Gamma}\left(\phi_i,\theta_i,\psi_i\right)=
    \begin{bmatrix}
    1&0&-\sin\theta_i\\
    0&\cos\phi_i&\cos\theta_i\sin\phi_i\\
    0&-\sin\phi_i&\cos\phi_i\cos\theta_i
    \end{bmatrix}
    .
\end{equation}

Angular acceleration of quadcopter $i\in \mathcal{V}$ is obtained by taking {\color{black}the} time derivative {\color{black}of} the angular velocity vector ${\color{black}\bf{\omega}}_i$ and related to control vector $\mathbf{u}_i$ by \cite{rastgoftar2021safe}:
\begin{equation}
\label{dwegai1}
\begin{split}
    \dot{{\color{black}\bf{\omega}}}_i=
    \tilde{\mathbf{B}}_{1,i} \begin{bmatrix}
    \mathbf{0}_{3\times 1}&\mathbf{I}_3
    \end{bmatrix}\mathbf{u}_i
    +\tilde{\mathbf{B}}_{2,i}.
\end{split}
\end{equation}
where
\begin{subequations}
\begin{equation}
    \tilde{\mathbf{B}}_{1,i}=\begin{bmatrix}
    \hat{\mathbf{i}}_{b,i}&\hat{\mathbf{j}}_{2,i}&\hat{\mathbf{k}}_{1,i}
    \end{bmatrix}
\end{equation}
\begin{equation}
   \tilde{\mathbf{B}}_{2,i}=\dot{\theta}_i\dot{\psi}_i\left(\hat{\mathbf{k}}_{1,i}\times \hat{\mathbf{j}}_{1,i} \right)+\dot{\phi}_{i}\left(\dot{\psi}_i\hat{\mathbf{k}}_{1,i}+\dot{\theta}_i\hat{\mathbf{j}}_{2,i}\right)\times \hat{\mathbf{i}}_{2,i}
\end{equation}
\end{subequations}
The rotational dynamics of quadcopter $i\in \mathcal{V}$ is given by \cite{zuo2010trajectory}
\begin{equation}
\label{dwegai2}
    \mathbf{J}_i\dot{\bf{\omega}}_i=
   -{\omega}_i\times \left(\mathbf{J}_i{\omega}_i\right)-{J}_{r,i}{\omega}_i\times\varpi_{i,r}\hat{\mathbf{k}}_{b,i}
   +\mathbf{T}_i
\end{equation}
where 
\begin{equation}
\label{Tiiiii}
\begin{split}
    \mathbf{T}_i=&\tau_{\phi,i}\hat{\mathbf{i}}_{b,i}+\tau_{\theta,i}\hat{\mathbf{j}}_{b,i}+\tau_{\psi,i}\hat{\mathbf{k}}_{b,i}=\tilde{\mathbf{B}}_{1,i}\begin{bmatrix}
    \tau_{\phi,i}&\tau_{\theta,i}&\tau_{\psi,i}
    \end{bmatrix}
    ^T
\end{split}
\end{equation}
is the quadcopter torque exerted on quadcopter $i\in \mathcal{V}$.
\begin{figure}[ht]
\centering
\includegraphics[width=3.3   in]{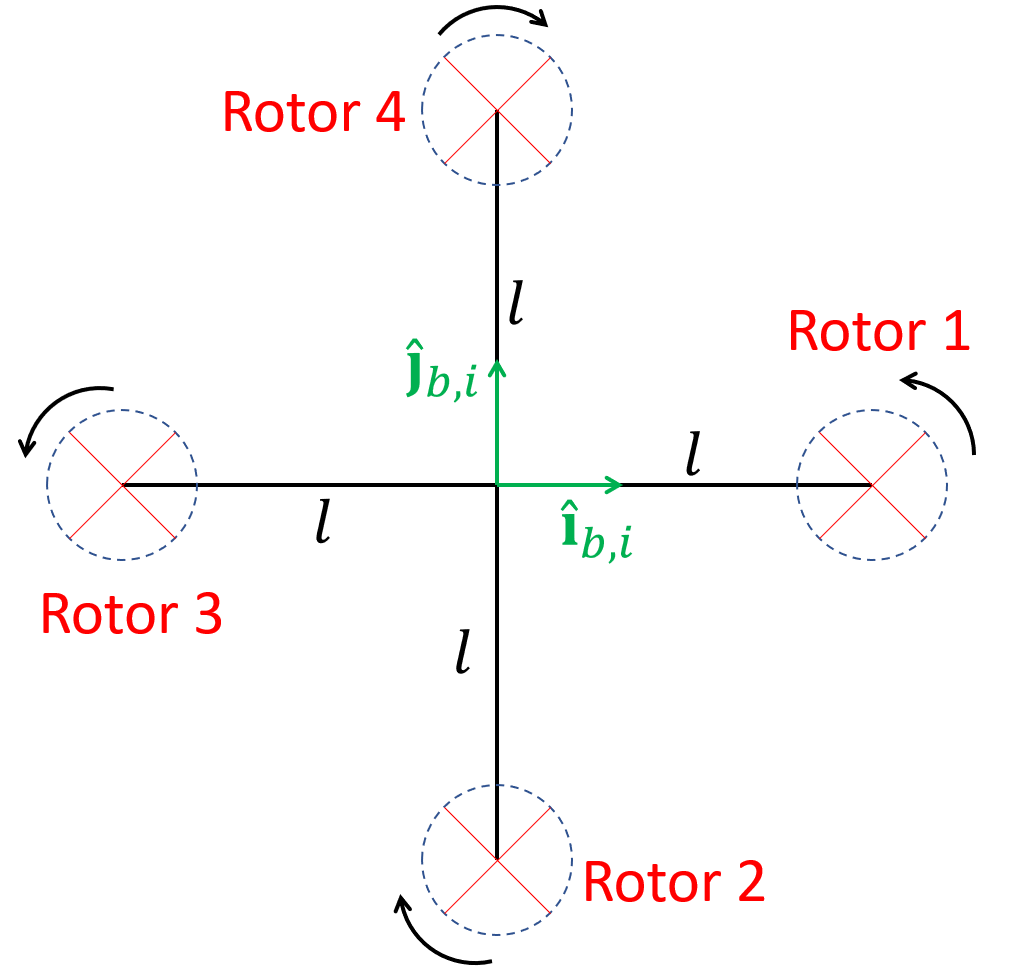}
\caption{Schematic of the plane of quadcopter $i\in \mathcal{V}$ defined by base vectors $\hat{\mathbf{i}}_{b,i}$ and $\hat{\mathbf{j}}_{b,i}$.
}
\label{RotationStandard}
\end{figure}
\subsection{Rotors' Angular Speeds}
The thrust force generated by rotor $j\in \left\{1,2,3,4\right\}$ of quadcopter $i\in \mathcal{V}$ is denoted by $p_{ij}$ and defined as follows:
\begin{equation}
    p_{ij}=b\varpi_{ij}^2,\qquad \forall i\in \mathcal{V},~j\in \left\{1,2,3,4\right\},
\end{equation}
where $b>0$ is the aerodynamic constant. The standard shown in  Fig. \ref{RotationStandard} is used to situate motors of quadcopter $i\in \mathcal{V}$, thus, components of torque $\mathbf{T}_i$, exerted on quadcopter $i\in \mathcal{V}$, are obtained as follows:
\begin{subequations}
\begin{equation}
    \tau_{\phi,i}=l\left(p_{i4}-p_{i2}\right)=bl\left(\varpi_{i4}^2-\varpi_{i2}^2\right),\qquad \forall i\in \mathcal{V},
\end{equation}
\begin{equation}
    \tau_{\theta,i}=l\left(p_{i3}-p_{i1}\right)=bl\left(\varpi_{i3}^2-\varpi_{i1}^2\right),\qquad \forall i\in \mathcal{V},
\end{equation}
\begin{equation}\label{taupsi}
    \tau_{\psi,i}=k\sum_{j=1}^4(-1)^j\varpi_{ij}^2,\qquad \forall i\in \mathcal{V},
\end{equation}
\end{subequations}
where $k$ is the aerodynamic constant in Eq. \eqref{taupsi}. Therefore, the rotors' angular speeds can be uniquely determined based on the thrust force ($p_i$) and control torque components ($\tau_{\phi,i}$, $\tau_{\theta,i}$, and $\tau_{\psi,i}$) by 
\begin{equation}
\label{mainangularspeed}
    \begin{bmatrix}
    p_i\\
    \tau_{\phi,i}\\
    \tau_{\theta,i}\\
    \tau_{\psi,i}
    \end{bmatrix}
    =
    \begin{bmatrix}
    b_i&b_i&b_i&b_i\\
    0&-b_il_i&0&b_il_i\\
    -b_il_i&0&b_il_i&0\\
    -k_i&k_i&-k_i&k_i
    \end{bmatrix}
    \begin{bmatrix}
    \varpi_{i1}^2\\
    \varpi_{i2}^2\\
    \varpi_{i3}^2\\
    \varpi_{i4}^2\\
    \end{bmatrix}
    ,
    \qquad \forall i\in \mathcal{V}.
\end{equation}
\begin{proposition}
Given  $p_i$, $\phi_i$, $\theta_i$, $\psi_i$, $\dot{p}_i$, $\dot{\phi}_i$, $\dot{\theta}_i$, $\dot{\psi}_i$, and $\mathbf{u}_i=\begin{bmatrix}
u_{p,i}&u_{\phi,i}&u_{\theta,i}&u_{\psi,i}
\end{bmatrix}^T$ at time $t\in \left[t_s,t_f\right]$, the angular speeds of rotors of quadcopter $i$ are determined by solving the following set of quadratic algebraic equations:
\begin{equation}
\label{quadraticmain}
    \mathbf{H}_{1,i}\begin{bmatrix}
    \varpi_{i1}^2\\
    \varpi_{i2}^2\\
    \varpi_{i3}^2\\
    \varpi_{i4}^2
    \end{bmatrix}
    +\mathbf{H}_{2,i}
    \begin{bmatrix}
    \varpi_{i1}\\
    \varpi_{i2}\\
    \varpi_{i3}\\
    \varpi_{i4}\\
    \end{bmatrix}
    +\mathbf{H}_{3,i}=\mathbf{0}_{4\times 1},\qquad \forall i\in \mathcal{V},
\end{equation}
where
\begin{subequations}
\begin{equation}
\mathbf{H}_{1,i}=
\begin{bmatrix}
1&\mathbf{0}_{1\times 3}\\
\mathbf{0}_{3\times 1}&\tilde{\mathbf{B}}_{1,i}
\end{bmatrix}
\begin{bmatrix}
    b_i&b_i&b_i&b_i\\
    0&-b_il_i&0&b_il_i\\
    -b_il_i&0&b_il_i&0\\
    -k_i&k_i&-k_i&k_i
    \end{bmatrix}
    \end{equation}
\begin{equation}
    \mathbf{H}_{2,i}=J_{r,i}\begin{bmatrix}
    0&0&0&0\\
    \omega_i\times \hat{\mathbf{k}}_{b,i}&-\omega_i\times \hat{\mathbf{k}}_{b,i}& \omega_i\times \hat{\mathbf{k}}_{b,i}&-\omega_i\times \hat{\mathbf{k}}_{b,i}\\
    \end{bmatrix}
\end{equation}
\begin{equation}
 \mathbf{H}_{3,i}=\begin{bmatrix}
    -1&\mathbf{0}_{1\times 3}\\
    \mathbf{0}_{3\times 1}&-\mathbf{B}_{1,i}
    \end{bmatrix}\mathbf{u}_i-\begin{bmatrix}
    0\\
    \mathbf{J}_i\tilde{\mathbf{B}}_{2,i}+{\omega}_i\times \left(\mathbf{J}_i{\omega}_i\right)
    \end{bmatrix}
    .
\end{equation}
\end{subequations}
\end{proposition}
\begin{proof}
By considering Eqs. \eqref{Tiiiii} and \eqref{mainangularspeed}, Eq. \eqref{mainangularspeed} can be rewritten as
\[
\begin{bmatrix}
\varpi_{i1}^2&\cdots&\varpi_{i1}^2
\end{bmatrix}
^T=\mathbf{H}_{1,i}\begin{bmatrix}
p_i&\mathbf{T}_i^T
\end{bmatrix}
^T,\qquad \forall i\in \mathcal{V}.
\]
By substituting $\dot{\omega}_i$ from Eq. \eqref{dwegai1}, $\mathbf{T}_i$ is obtained as follows:
\begin{equation}
\label{88badeq}
    \mathbf{T}_i=\mathbf{B}_{1,i}\begin{bmatrix}
    \mathbf{0}_{3\times 1}&\mathbf{I}_3
    \end{bmatrix}\mathbf{u}_i+\mathbf{B}_{2,i},\qquad \forall i\in \mathcal{V},
\end{equation}
where
\begin{subequations}
\begin{equation}
    \mathbf{B}_{1,i}=\mathbf{J}_i\tilde{\mathbf{B}}_{1,i},
\end{equation}
\begin{equation}
\begin{split}
    \mathbf{B}_{2,i}=&\mathbf{J}_i\tilde{\mathbf{B}}_{2,i}+{\omega}_i\times \left(\mathbf{J}_i{\omega}_i\right)+{J}_{r,i}{\omega}_i\times\varpi_{i,r}\hat{{\mathbf{k}}}_{b,i}\\
    =&\mathbf{J}_i\tilde{\mathbf{B}}_{2,i}+{\omega}_i\times \left(\mathbf{J}_i{\omega}_i\right)-\mathbf{H}_{2,i}\begin{bmatrix}
\varpi_{i1}&\cdots&\varpi_{i1}
\end{bmatrix}
^T.
\end{split}
\end{equation}
\end{subequations}
By substituting $\mathbf{T}_i$ obtained in  Eq. \eqref{88badeq}, the angular speeds of the rotors of quadcopter $i\in \mathcal{V}$ are assigned by Eq. \eqref{quadraticmain}.
\end{proof}

\subsection{Quadcopter Trajectory Control}
\label{Quadcopter Trajectory Control}
In this section, we use the feedback linization control method to design trajectory control $\mathbf{u}_i$ for every quadcopter $i\in \mathcal{V}$. To this end, we provide Definition \ref{liederivative} to formally define Lie derivative before proceeding.

\begin{definition}\label{liederivative}
Let $y:\mathbb{R}^{p}\rightarrow \mathbb{R}$ and $\mathbf{f}:\mathbb{R}^{p}\rightarrow \mathbb{R}^p$ be smooth functions. The Lie derivative $y$ with respect to $\mathbf{f}$ is defined as follows:
\[
L_{\mathbf{f}}y=\bigtriangledown y \mathbf{f}.
\]
\end{definition}

We define state transformation $\mathbf{z}_i=\rightarrow\left(\mathbf{z}_i,\mathbf{y}_i\right)$ given by
\begin{equation}
    \mathbf{z}_i=\begin{bmatrix}\mathbf{r}_i^T&\dot{\mathbf{r}}_i^T&\ddot{\mathbf{r}}_i^T&\dddot{\mathbf{r}}_i^T&\psi_i&\dot{\psi}_i\end{bmatrix}^T
\in \mathbb{R}^{14\times 1},\qquad \forall i\in \mathcal{V},
\end{equation}
where $\mathbf{z}_i$ is updated by the following linear-time-invariant dynamics:
\begin{equation}\label{lineardynamicsfel}
    \dot{\mathbf{z}}_i=\mathbf{A}_{\mathrm{SF}}\mathbf{z}_i+\mathbf{B}_{\mathrm{SF}}\mathbf{v}_i,\qquad \forall i\in \mathcal{V},
\end{equation}
and
\begin{subequations}
\begin{equation}
    \mathbf{A}_{\mathrm{SF}}=\begin{bmatrix}
    \mathbf{0}_{9\times 3}&\mathbf{I}_3&\mathbf{0}_{9\times 1}&\mathbf{0}_{9\times 1}\\
    \mathbf{0}_{3\times 3}&\mathbf{0}_{3\times 9}&\mathbf{0}_{3\times 1}&\mathbf{0}_{3\times 1}\\
    \mathbf{0}_{1\times 3}&\mathbf{0}_{1\times 9}&0&1\\
    \mathbf{0}_{1\times 3}&\mathbf{0}_{1\times 9}&0&0\\
    \end{bmatrix}
    ,
   \end{equation}
    \begin{equation}
    \mathbf{B}_{\mathrm{SF}}=\begin{bmatrix}
    \mathbf{0}_{9\times 3}&\mathbf{0}_{9\times 1}\\
    \mathbf{I}_{3}&\mathbf{0}_{3\times 1}\\
    \mathbf{0}_{1\times 3}&0\\
    \mathbf{0}_{1\times 3}&1\\
    \end{bmatrix}
    .
\end{equation}
\end{subequations}

Here, $\mathbf{v}_i$ is related to the control input of quadcopter $i\in \mathcal{V}$, denoted by $\mathbf{u}_i$, by
\vspace{-0.3cm}
\begin{equation}
    \mathbf{v}_i=\mathbf{M}_{1,i}\mathbf{u}_i+\mathbf{M}_{2,i},
\end{equation}
where
\begin{subequations}
\vspace{-0.2cm}
\begin{equation}
  \mathbf{M}_{1,i}=  \begin{bmatrix}
        L_{{\mathbf{g}}_{_{1}}}L_{{\mathbf{f}}}^{3}x_i& L_{{\mathbf{g}}_{_{2}}}L_{{\mathbf{f}}}^{3}x_i& L_{{\mathbf{g}}_{_{3}}}L_{{\mathbf{f}}}^{3}x_i& L_{{\mathbf{g}}_{_{4}}}L_{{\mathbf{f}}}^{3}x_i\\
        L_{{\mathbf{g}}_{_{1}}}L_{{\mathbf{f}}}^{3}y_i& L_{{\mathbf{g}}_{_{2}}}L_{{\mathbf{f}}}^{3}y_i& L_{{\mathbf{g}}_{_{3}}}L_{{\mathbf{f}}}^{3}y_i& L_{{\mathbf{g}}_{_{4}}}L_{{\mathbf{f}}}^{3}y_i\\
        L_{{\mathbf{g}}_{_{1}}}L_{{\mathbf{f}}}^{3}z_i& L_{{\mathbf{g}}_{_{2}}}L_{{\mathbf{f}}}^{3}z_i& L_{{\mathbf{g}}_{_{3}}}L_{{\mathbf{f}}}^{3}z_i& L_{{\mathbf{g}}_{_{4}}}L_{{\mathbf{f}}}^{3}z_i\\
         L_{{\mathbf{g}}_{_{1}}}L_{{\mathbf{f}}}\psi_i& L_{{\mathbf{g}}_{_{2}}}L_{{\mathbf{f}}}\psi_i& L_{{\mathbf{g}}_{_{3}}}L_{{\mathbf{f}}}\psi_i& L_{{\mathbf{g}}_{_{4}}}L_{{\mathbf{f}}}\psi_i\\
        \end{bmatrix}
        \in \mathbb{R}^{14\times 14}
        ,
\end{equation}
\vspace{-0.2cm}
\begin{equation}
  \mathbf{M}_{2,i}=  \begin{bmatrix}
        L_{{\mathbf{f}}}^{4}x_i&
       L_{{\mathbf{f}}}^{4}y_i&
        L_{{\mathbf{f}}}^{4}z_i&
         L_{{\mathbf{f}}}^2\psi_i
        \end{bmatrix}
        ^T\in \mathbb{R}^{14\times 1}
        .
\end{equation}
\end{subequations}
The control design objective is to choose $\mathbf{u}_i$ such that $\mathbf{y}_i$ stably tracks desired output $\mathbf{y}_{i,d}=\begin{bmatrix}
\mathbf{p}_{i}^T&\psi_{i,d}
\end{bmatrix}$ where $\mathbf{p}_{i}$  and $\psi_{i,d}$ are the global desired trajectory and desired yaw angle of quadcopter $i\in \mathcal{V}$. Without loss of generality, this paper assumes that $\psi_{i,d}(t)=0$ at any time $t$. To achieve the control objective, we define desired state vector
\begin{equation}\label{zid}
    \mathbf{z}_{i,d}=\begin{bmatrix}
    \mathbf{r}_{i,d}^T&\dot{\mathbf{r}}_{i,d}^T&\ddot{\mathbf{r}}_{i,d}^T&\dddot{\mathbf{r}}_{i,d}^T&\psi_{i,d}&\dot{\psi}_{i,d}
    \end{bmatrix}
    ,\qquad \forall i\in \mathcal{V},
\end{equation}
and choose
\vspace{-0.2cm}
\begin{equation}\label{vi}
    \mathbf{v}_i=\mathbf{K}_{i}\left(\mathbf{z}_{i,d}-\mathbf{z}_i\right),\qquad \forall i\in \mathcal{V},
\end{equation}
such that $\mathbf{A}_{\mathrm{SF}}-\mathbf{B}_{\mathrm{SF}}\mathbf{K}_{i}$ is Hurwitz. Then, the control input of quadcopter $i\in \mathcal{V}$ is obtained by
\begin{equation}
    \mathbf{u}_i=\mathbf{M}_{1,i}^{-1}\left(\mathbf{v}_i-\mathbf{M}_{2,i}\right). 
\end{equation}

\begin{theorem}
Assume $\mathbf{z}_{i,d}$, defined by \eqref{zid}, is a bounded input, $\mathbf{v}_i$ is selected by \eqref{vi}, and control gain matrix $\mathbf{K}_i$ is selected such that $\mathbf{A}_{\mathrm{SF}}-\mathbf{B}_{\mathrm{SF}}\mathbf{K}_{i}$ is Hurwitz. Then, there exists a unique $t_f^*>t_s$ such that safety conditions  \eqref{quadcopterpropeller} and \eqref{boundedness} are satisfied by choosing any $t_f\geq t_f^*$.
\end{theorem}
\begin{proof}
By substituting $\mathbf{v}_i$ from Eq. \eqref{vi}, Eq. \eqref{lineardynamicsfel} simplifies to
\begin{equation}\label{proofeq}
    \dot{\mathbf{z}}_i=\left(\mathbf{A}_{\mathrm{SF}}-\mathbf{B}_{\mathrm{SF}}\mathbf{K}_{i}\right){\mathbf{z}}_i+\mathbf{B}_{\mathrm{SF}}\mathbf{K}_{i}{\mathbf{z}}_{i,d},\qquad \forall i\in \mathcal{V}.
\end{equation}
If $\mathbf{A}_{\mathrm{SF}}-\mathbf{B}_{\mathrm{SF}}\mathbf{K}_{i}$ is Hurwitz and ${\mathbf{z}}_{i,d}$ is bounded, then, dynamics is Bounded Input Bounded Output (BIBO) stable which in turn implies that $\mathbf{z}_i(t)$ remains bounded at any time $t$.  
Now, we define $\mathbf{E}_{i}=\begin{bmatrix}
\left(\mathbf{z}_i-\mathbf{z}_{d,i}\right)^T&\psi_i&\dot{\psi}_i
\end{bmatrix}^T$ as the error, and  obtain the following error dynamics:
\begin{equation}
    \dot{\mathbf{E}}_i=\left(\mathbf{A}_{\mathrm{SF}}-\mathbf{B}_{\mathrm{SF}}\mathbf{K}_{i}\right){\mathbf{E}}_i+\begin{bmatrix}\mathbf{0}_{3\times 9}\\\mathbf{I}_3\\\mathbf{0}_{3\times 2}\end{bmatrix}\ddddot{\mathbf{p}}_i
\end{equation}
We say that  $\ddddot{\mathbf{p}}_i(t)\rightarrow \mathbf{0}$, if $\left(t_f-t_s\right)\rightarrow\infty$. Therefore, there exists a final time $t_f'$ such that safery condition \eqref{boundedness} is satisfied for every quadcopter $i\in \mathcal{V}$. Also, $\dot{\mathbf{p}}_i(t)$, $\ddot{\mathbf{p}}_i(t)$,  $\dddot{\mathbf{p}}_i(t)$, and $\ddddot{\mathbf{p}}_i(t)$ are decreased at any time $t\in \left[t_s,t_f\right]$, if $\left(t_f-t_s\right)\rightarrow\infty$. Therefore, there exists a final time $t_f^{''}$ such that the angular speeds of rotors of every quadcopter $i\in \mathcal{V}$ satisfy safety condition \eqref{quadcopterpropeller}. Therefore, safety conditions \eqref{quadcopterpropeller} and \eqref{boundedness} are both satisfied if we choose a final time $t_f\geq t_f^*$, where $t_f^*=\max\left\{t_f^{'},t_f^{''}\right\}$.
\end{proof}

\begin{table}
\centering
\caption{Parameters of quadcopter models used for simulation. The quadcopter paramters are selected from \eqref{xufg}.}
    \begin{tabular}{|c|c|c|}
    \hline
         Parameter ($ \forall i\in \mathcal{V}$)&Value  &Unit\\
         \hline
         $m_i=m$&$0.5$  &$kg$\\
         $g$&  $9.81$ &$m/s^2$\\
         $l_i=l$&  $0.25$  &$m$\\
         $J_{r,i}=J_r$& $3.357\times 10^{-5}$    &$kg~m^2$\\
         $J_{x,i}=J_x$& $0.0196$ &$kg~m^2$\\
         $J_{y,i}=J_y$&$0.0196 $  &$kg~m^2$\\
         $J_{z,i}=J_z$&$0.0264$ &$kg~m^2$\\
         $b_i=b$& $3\times 10^{-5}$ &$N~s^2/rad^2$\\
         $k_i=k$&$1.1\times 10^{-6}$&$N~s^2/rad^2$\\
         \hline
    \end{tabular}
    \label{quadparameters}
\end{table}
\begin{figure}[ht]
\centering
\includegraphics[width=3.3   in]{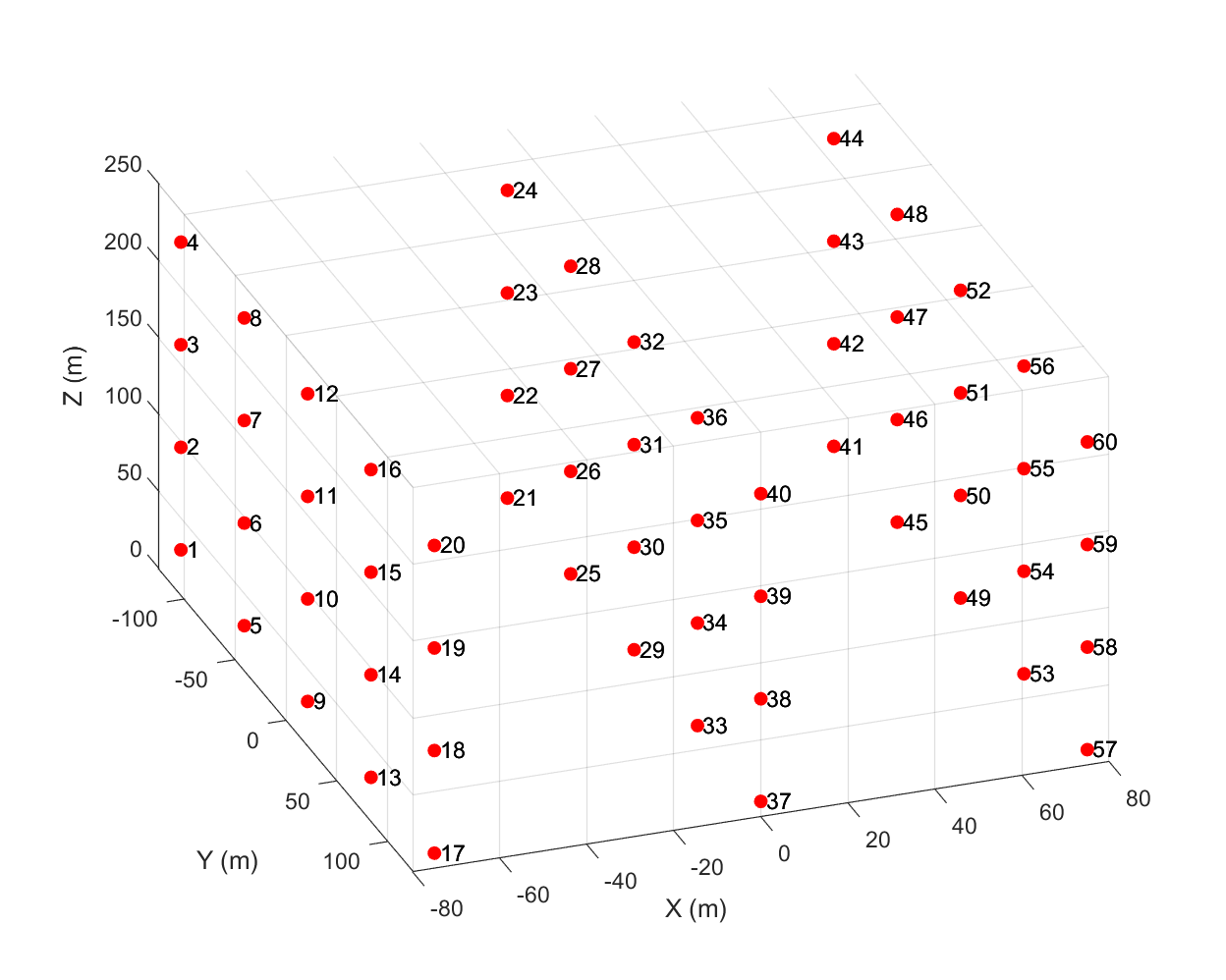}
\caption{Initial configuration of the quadcopter team forming a cuboid in the motion space.
}
\label{initial}
\end{figure}

\begin{figure}[ht]
\centering
\includegraphics[width=3.3   in]{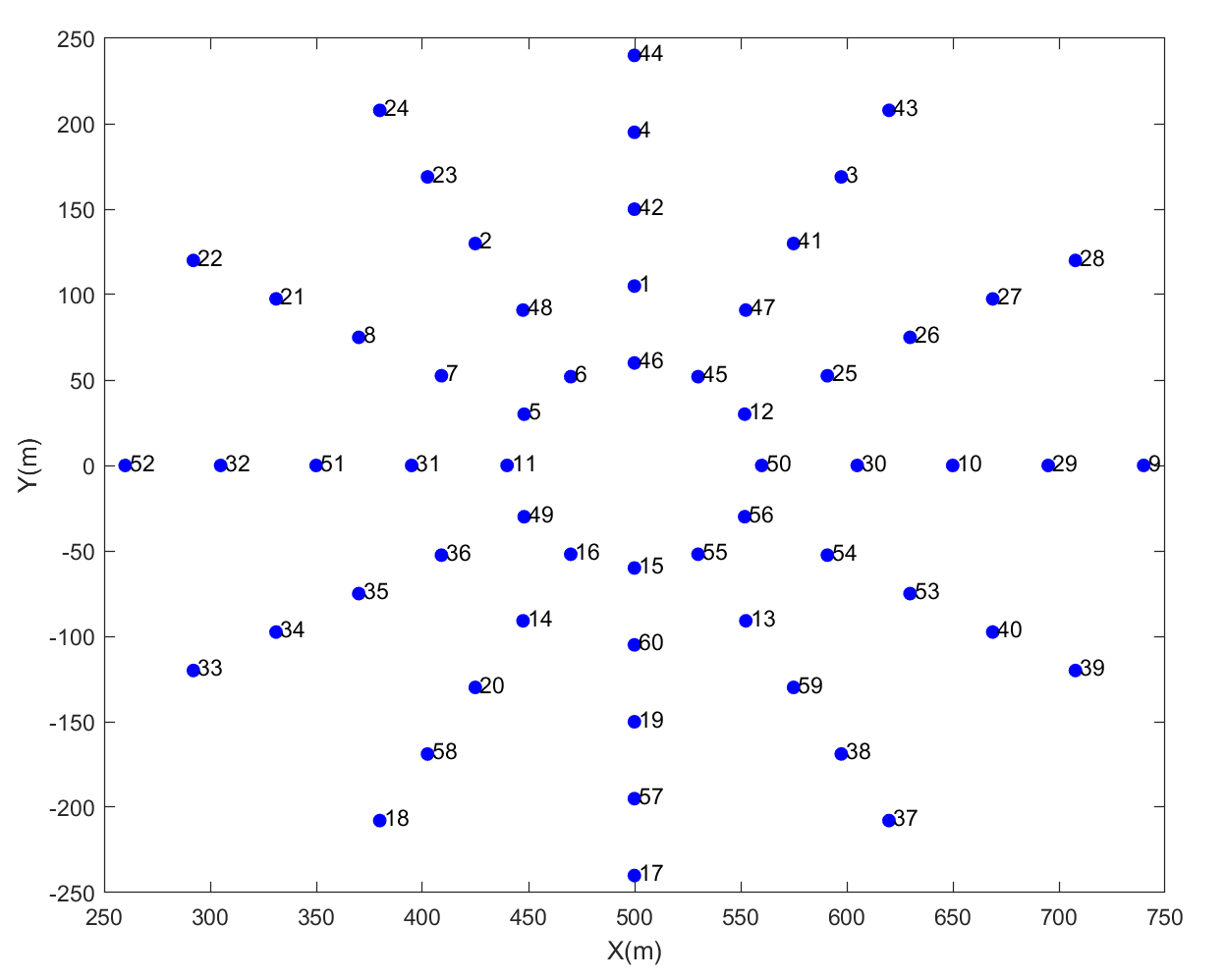}
\caption{Final configuration of the quadcopter team forming a disk in the $x-y$ plane.
}
\label{final}
\end{figure}

\begin{figure}
 \centering
 \subfigure[$\varpi_{i,1}~\forall i\in \mathcal{V}$]{\includegraphics[width=0.9\linewidth]{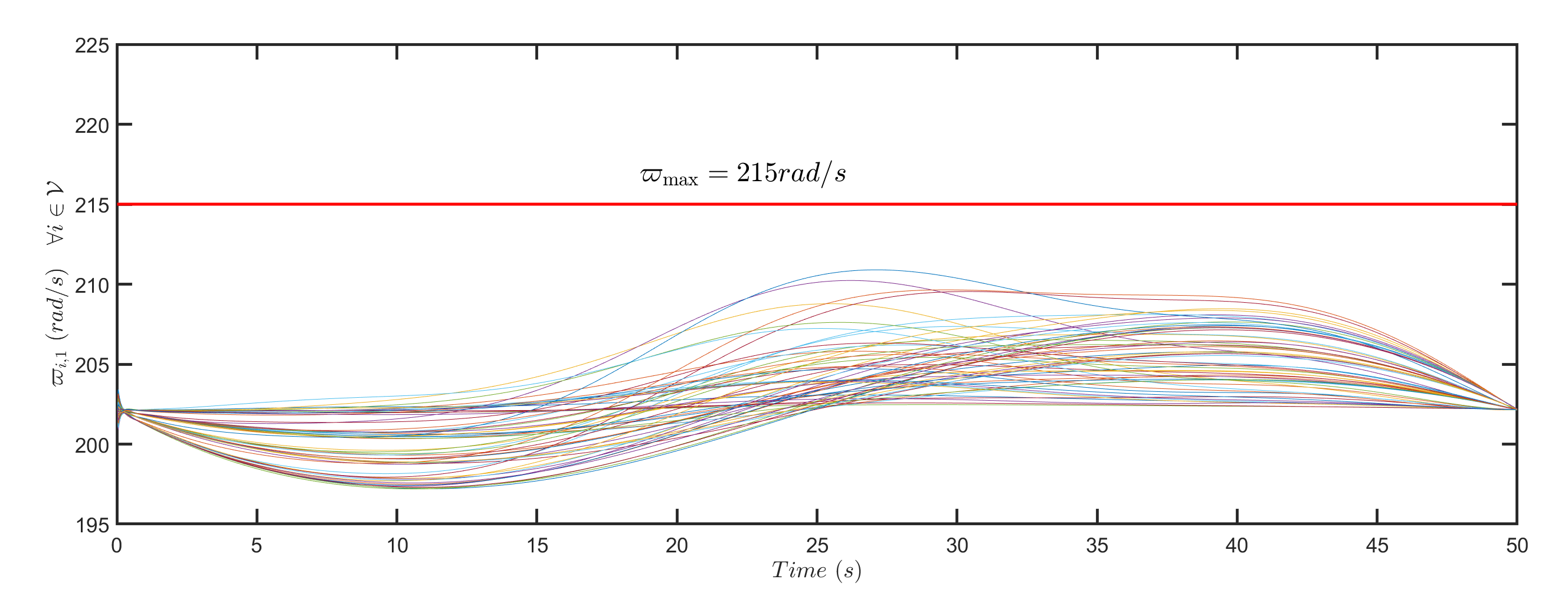}}
  \subfigure[$\varpi_{i,2}~\forall i\in \mathcal{V}$]{\includegraphics[width=0.9\linewidth]{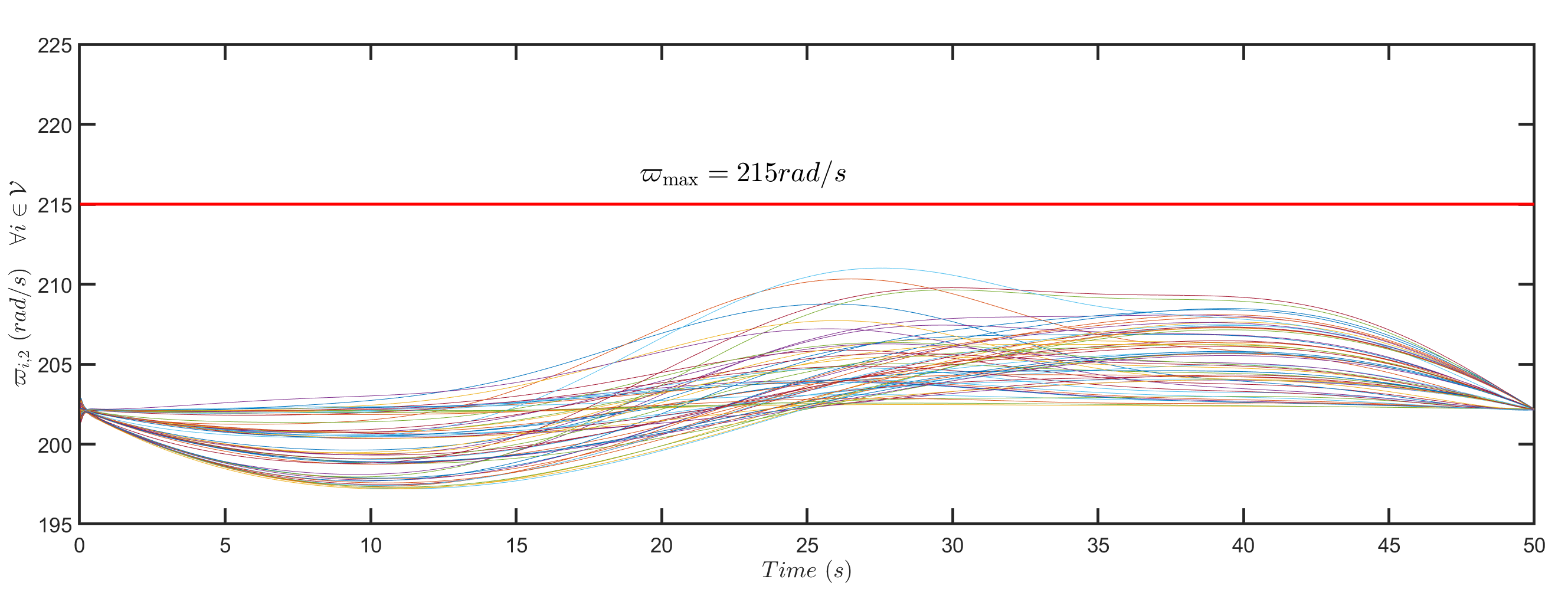}}
   \subfigure[$\varpi_{i,3}~\forall i\in \mathcal{V}$]{\includegraphics[width=0.9\linewidth]{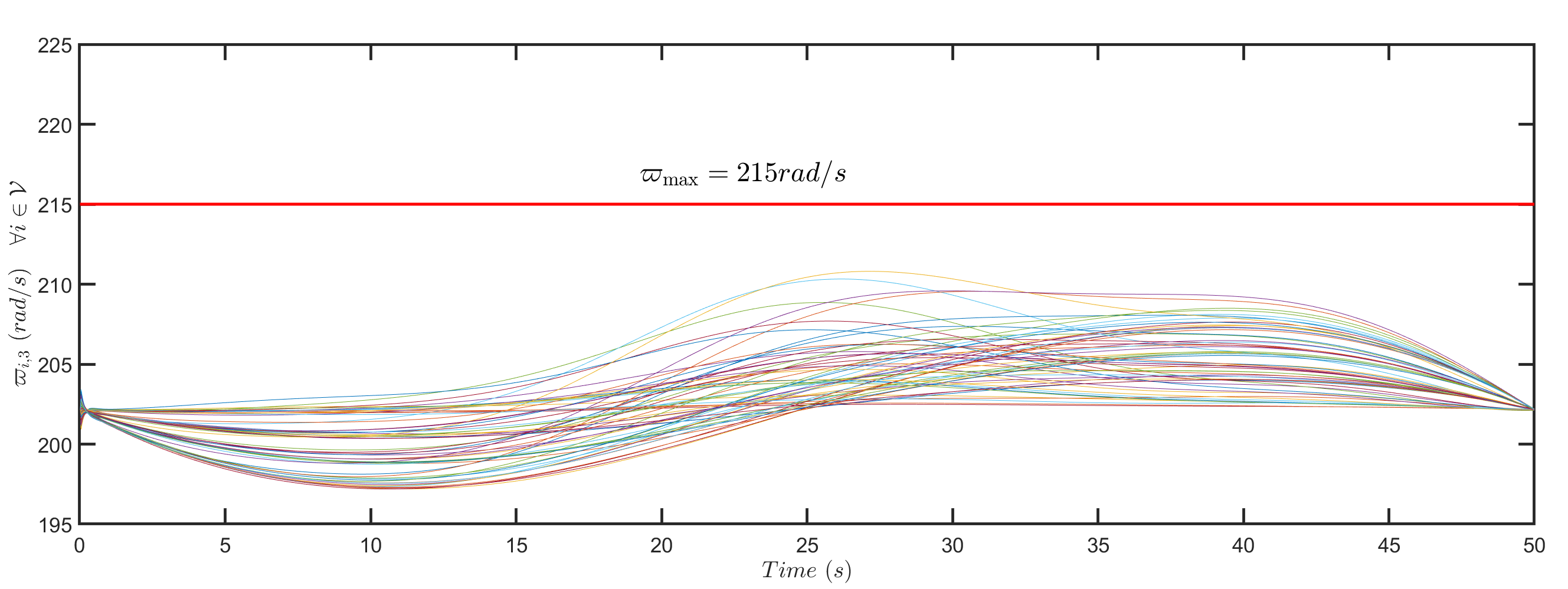}}
  \subfigure[$\varpi_{i,3}~\forall i\in \mathcal{V}$]{\includegraphics[width=0.9\linewidth]{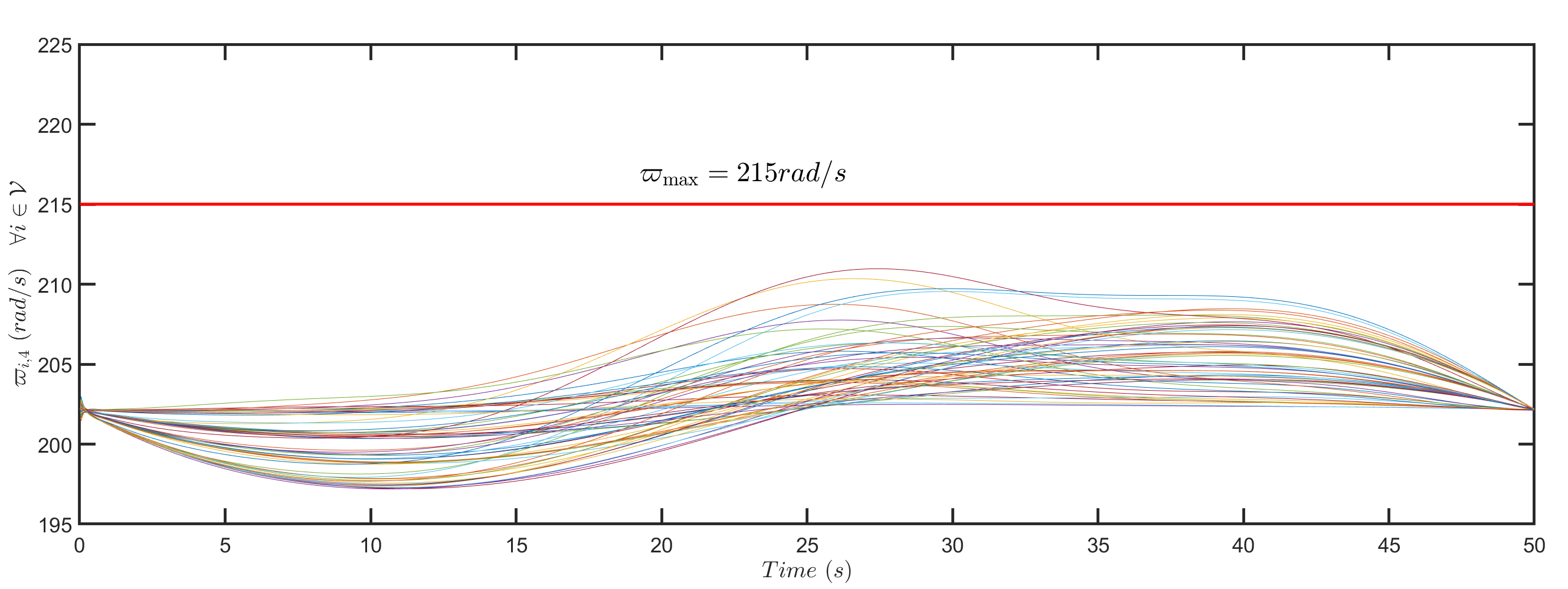}}
\vspace{-0.3cm}
     \caption{Angular speeds of rotors $1$ through four of all quadcopters. It is seen that the safety constraint \eqref{quadcopterpropeller} of every quadcopter $i\in \mathcal{V}$ is satisfied where $\varpi_{\mathrm{max}}=220~rad/s$.}
\label{AngularSpeeds}
\end{figure}
\begin{figure}
 \centering
 \subfigure[$x$ components of all quadcopters]{\includegraphics[width=0.9\linewidth]{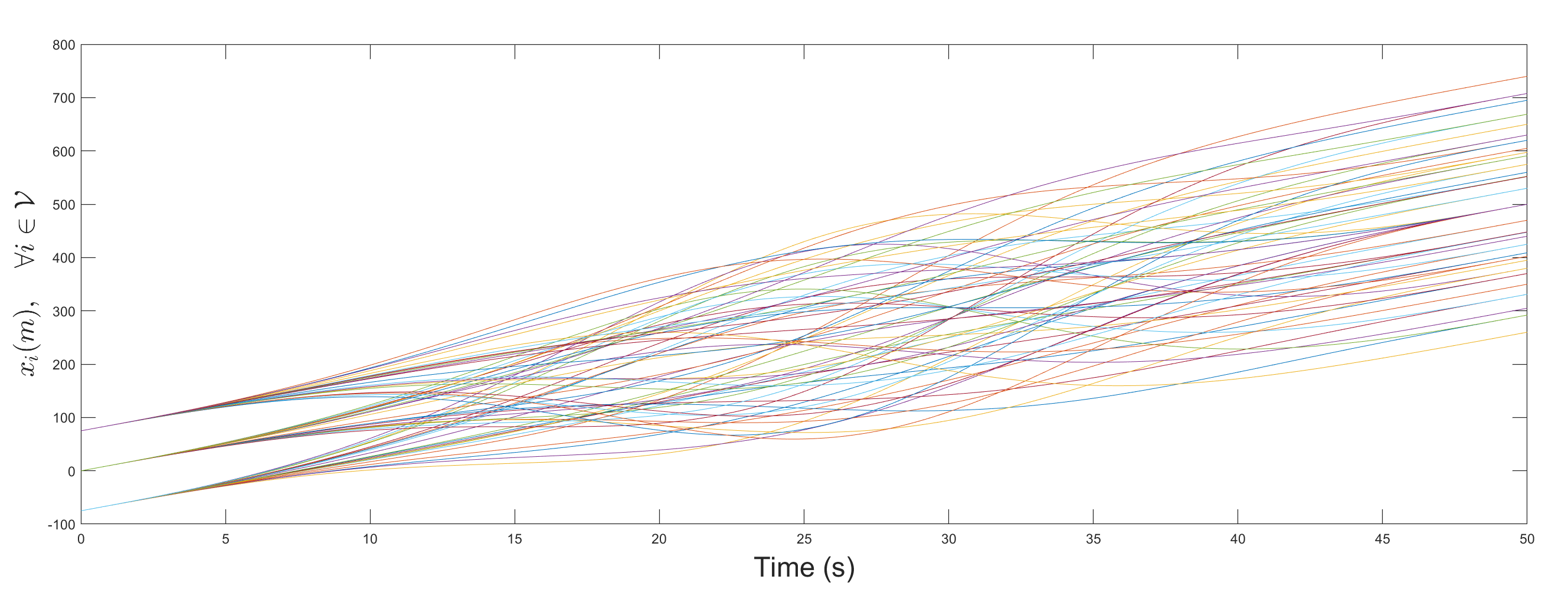}}
  \subfigure[$y$ components of all quadcopters]{\includegraphics[width=0.9\linewidth]{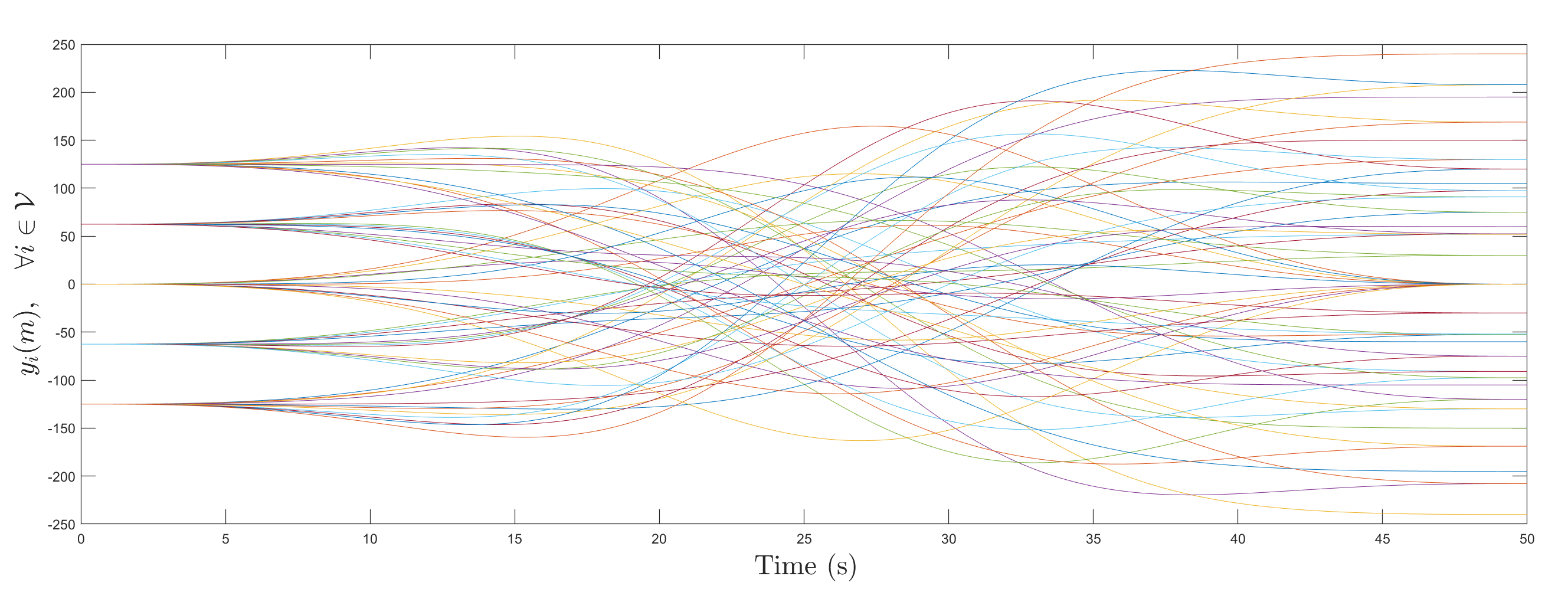}}
   \subfigure[$z$ components of all quadcopters]{\includegraphics[width=0.9\linewidth]{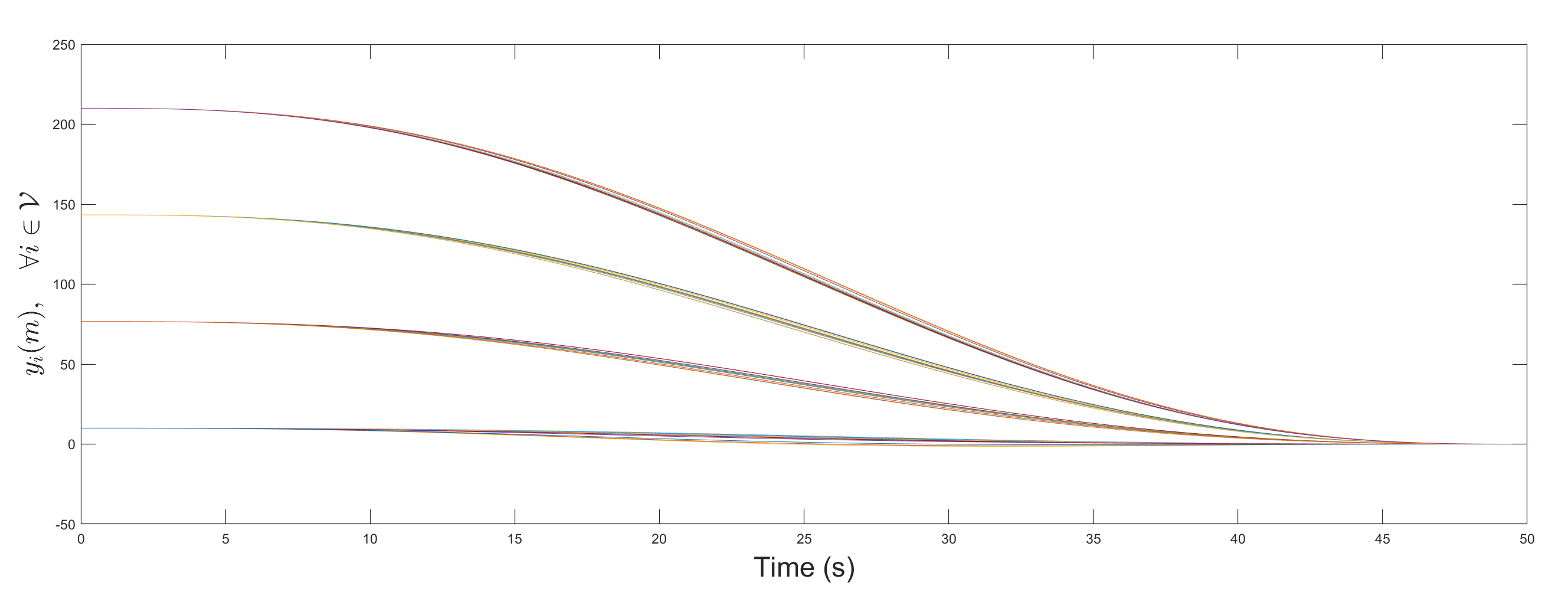}}
\vspace{-0.3cm}
     \caption{Position components of all quadcopters versus time for $t\in \left[0,50\right]s$.}
\label{XYZ}
\end{figure}
\section{Simulation Results}\label{Simulation Results}
We consider evolution of an MQS consisting of $60$ quadcopters where quadcopters have the same characteristics  and are all modeled by dynamics \eqref{NonlinearDynamics} with $\mathbf{x}_i$, $\mathbf{u}_i$, $\mathbf{f}\left(\mathbf{x}_i\right)$, and $\mathbf{g}\left(\mathbf{x}_i\right)$ given in \eqref{xufg}. We use the quadcopter parameters presented in Ref. \cite{gopalakrishnan2017quadcopter} and listed in Table  \ref{quadparameters} to simulate the real-time deployment coordination of the MQS from an initial formation shown in Fig. \ref{initial} to the final configuration shown in Fig. \ref{final}. Given the initial and final configurations of the MQS, $d_{\mathrm{min}}\beta^*=1.1889$. Therefore,   $\delta=0.19$ assign the upper-bound for the RTD tracking error. For simulation, we assume that  every quadcopter can be enclosed by a ball of radius $\epsilon=0.40$, $\varpi_{\mathrm{max}}=215~rad/s$ is the upper limit for the quadcopters' angular speeds.  

We further assume that the MQS moves with velocity $10\hat{\mathbf{e}}_1m/s$ before and after  RTD is activated, i.e $\dot{\mathbf{d}}(0)=\dot{\mathbf{d}}\left(t_f\right)=10 m/s$. Fig. \ref{AngularSpeeds} plots angular speeds of all quadcopter rotors. As it is seen safety condition \eqref{quadcopterpropeller} is satisfied for every quadcopter $i\in \mathcal{V}$. Fig. \ref{XYZ} plots $x$, $y$, and $z$ components of actual positions of all quadcopters versus time for $t\in \left[0,50\right]s$. 

\section{Conclusion}\label{Conclusion}
This paper developed a novel physics-based solution for the real-time deployment of multi-agent systems between arbitrary moving configurations. The proposed approach decomposes the RTD into rigid-boday rotation, $1$-D homogeneous transformation, and $2$-D heterogeneous motion.    Without loss of generality, we assumed that each agent is a quadcopter modeled by a $14$-th order nonlinear dynamics, and applied the feedback linearization control for each quadcopter  to stably and safely track the  desired RTD trajectory. By choosing a sufficiently-large RTD travel time, we assured that the safety constraints, including bounded  rotor speeds conditions and inter-agent collision avoidance, are assured.


\bibliographystyle{IEEEtran}
\bibliography{references}

\begin{IEEEbiography}[{\includegraphics[width=1in,height=1.25in,clip,keepaspectratio]{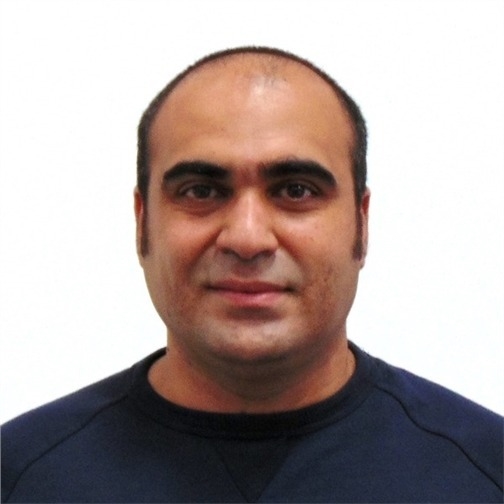}}]
{\textbf{Hossein Rastgoftar}} is an Assistant Professor in the Department of Aerospace and Mechanical Engineering at the University of Arizona and an Adjunct Assistant Professor at the Department of Aerospace Engineering at the University of Michigan Ann Arbor. He received the B.Sc. degree in mechanical engineering-thermo-fluids from Shiraz University, Shiraz, Iran, the M.S. degrees in mechanical systems and solid mechanics from Shiraz University and the University of Central Florida, Orlando, FL, USA, and the Ph.D. degree in mechanical engineering from Drexel University, Philadelphia, in 2015. 
\end{IEEEbiography}

\end{document}